\documentclass{article}
\UseRawInputEncoding

\PassOptionsToPackage{numbers, sort&compress}{natbib}
\usepackage[final, main]{neurips_2026} % Use [final] for camera-ready
\usepackage[utf8]{inputenc} % allow utf-8 input
\usepackage[T1]{fontenc}    % use 8-bit T1 fonts
\usepackage{hyperref}       % hyperlinks
\usepackage{url}            % simple URL typesetting
\usepackage{booktabs}       % professional-quality tables
\usepackage{amsfonts}       % blackboard math symbols
\usepackage{nicefrac}       % compact symbols for 1/2, etc.
\usepackage{microtype}      % microtypography
\usepackage{xcolor}         % colors

\usepackage{amsmath, amssymb, amsthm}

\usepackage{graphicx}
\usepackage{enumitem}
\usepackage{newunicodechar}
\usepackage{subcaption}
\usepackage{algorithm}
\usepackage{algpseudocode}
\usepackage{booktabs}
\usepackage{listings}
\usepackage{amsthm}
\newtheorem{lemma}{Lemma}[section]
\newtheorem{proposition}{Proposition}[section]
\usepackage{makecell}

\theoremstyle{definition} 
\newunicodechar{ℝ}{\ensuremath{\mathbb{R}}}
\newunicodechar{ℕ}{\ensuremath{\mathbb{N}}}
\newunicodechar{ℤ}{\ensuremath{\mathbb{Z}}}
\newunicodechar{ε}{\ensuremath{\varepsilon}}
\newunicodechar{∀}{\ensuremath{\forall}}
\newunicodechar{∃}{\ensuremath{\exists}}
\newunicodechar{∈}{\ensuremath{\in}}
\newunicodechar{∉}{\ensuremath{\notin}}
\newunicodechar{∧}{\ensuremath{\wedge}}
\newunicodechar{∨}{\ensuremath{\vee}}
\newunicodechar{¬}{\ensuremath{\neg}}
\newunicodechar{→}{\ensuremath{\to}}
\newunicodechar{⇒}{\ensuremath{\Rightarrow}}
\newunicodechar{↦}{\ensuremath{\mapsto}}
\newunicodechar{≤}{\ensuremath{\leq}}
\newunicodechar{≥}{\ensuremath{\geq}}
\newunicodechar{≠}{\ensuremath{\neq}}
\newunicodechar{≈}{\ensuremath{\approx}}
\newunicodechar{≡}{\ensuremath{\equiv}}
\newunicodechar{⊆}{\ensuremath{\subseteq}}
\newunicodechar{ᵀ}{\ensuremath{^{\top}}}
\newunicodechar{ᵥ}{\ensuremath{{}_{v}}}
\newunicodechar{•}{\ensuremath{\bullet}}
\newunicodechar{·}{\ensuremath{\cdot}}
\newunicodechar{⊂}{\ensuremath{\subset}}
\newunicodechar{⊃}{\ensuremath{\supset}}
\newunicodechar{∪}{\ensuremath{\cup}}
\newunicodechar{∩}{\ensuremath{\cap}}
\newunicodechar{⊙}{\ensuremath{\odot}}
\newunicodechar{⊗}{\ensuremath{\otimes}}
\newunicodechar{⊕}{\ensuremath{\oplus}}
\newunicodechar{×}{\ensuremath{\times}}
\newunicodechar{÷}{\ensuremath{\div}}
\newunicodechar{√}{\ensuremath{\sqrt{\phantom{x}}}}
\newunicodechar{∞}{\ensuremath{\infty}}
\newunicodechar{†}{\ensuremath{^\dagger}}
\newunicodechar{δ}{\ensuremath{\delta}}

\newcommand{\cvector}{\mathbf{c}_t}
\newcommand{\marksystem}{MaRK}

\title{\marksystem{}: Markov-adapted Recurrent Kernels for Dynamic Operator Conditioning in State Space Models}

\author{%
  \text{Syed Ibrahim Omer\thanks{Corresponding Author: ibrahsyed2-c@my.cityu.edu.hk}}\\  Department of Data Science \\
  City University of Hong Kong \\
  \And
  Ginny Y. Wong \\
  NVIDIA AI Technology Center \\
  NVIDIA, Santa Clara, USA \\
  \And
  Xiangyu Zhao \\
  Department of Data Science \\
  City University of Hong Kong \\
}

\date{}

\begin{document}

\maketitle

\begin{abstract}
State Space Models (SSMs) offer an efficient alternative to Transformers for sequence modeling, yet conditioning pre-trained SSMs for iterative generation typically operates outside the recurrent operator, through input injection or activation modulation. While such mechanisms expose the model to conditioning information, they leave the underlying temporal dynamics fixed. We introduce \marksystem{} (Markov-adapted Recurrent Kernels), a dynamic operator-conditioning framework that maps context vectors directly into bounded modulations of a frozen SSM's recurrence ($A$), read-in ($B$), read-out ($C$), skip ($D$), and discretization ($\Delta$) parameters. Viewed through the lens of Linear Parameter-Varying systems, \marksystem{} induces a context-indexed family of Markov parameter sequences, allowing each diffusion timestep to reshape the model's input-output memory kernel. We instantiate \marksystem{} on a frozen 111M-parameter Hydra SSM backbone and study three adapter geometries: Hypernet, Chebyshev polynomial, and Discrete Cosine Transform kernels. Since these adapters modify the Markov parameter sequence through low-rank auxiliary maps on the frozen backbone, parameter-efficient fine-tuning arises as a structural consequence of the adaptation mechanism itself, requiring only 6.3--11M trainable auxiliary parameters to transition from a bidirectional objective to an iterative diffusion regime. The bounded recurrence parameterization further yields an analytic Affine Quadratic Stability certificate for the modulated recurrence. Through synthetic LPV recovery experiments and Markov-operator diagnostics, we show that \marksystem{} recovers coordinate-invariant temporal operators under matched assumptions and produces distinct, stable timestep-conditioned memory profiles. Empirically, the Chebyshev variant yields the strongest performance, achieving an average validation loss of 2.55, followed by the DCT (2.59) and Hypernet (3.77) geometries. Together, these results provide initial evidence that dynamic operator modulation is a principled operator-level conditioning mechanism for adapting SSMs beyond input-stream injection and adaptive normalization.
\end{abstract}

\section{Introduction} \label{sec:introduction}

State Space Models (SSMs) have emerged as highly efficient, continuous-time alternatives to Transformers for large-scale sequence modeling \cite{gu2024mambalineartimesequencemodeling, dao2024transformersssmsgeneralizedmodels, lahoti2026mamba3improvedsequencemodeling}. Conditional sequence generation (e.g., discrete diffusion) typically relies on post-hoc interventions \cite{ho2020denoisingdiffusionprobabilisticmodels, austin2021d3pm, xiang2022diffusion, nie2025largelanguagediffusionmodels, ye2025dream7bdiffusionlarge} via two external paradigms: \textbf{Input-Stream Injection} \cite{ho2020denoisingdiffusionprobabilisticmodels}, which concatenates context vectors into the input but forces state-space dynamics to carry context across the entire sequence, and \textbf{Adaptive Layer Normalization (AdaLN)} \cite{peebles2023scalablediffusionmodelstransformers, nathan2024diffusionmodelswithoutattention, phung2024dimsum}. While AdaLN dynamically modulates layer activations, it operates externally to the core SSM operator without parameterizing the underlying temporal dynamics ($A, B, C, \Delta$). Consequently, neither paradigm directly parameterizes the Markov parameter sequence ($\mathcal{H}(c_t)$), which fully characterizes an SSM's input-output temporal behavior.

To bridge this gap, we introduce \textbf{\marksystem{} (Markov-adapted Recurrent Kernels)}, an operator-level conditioning approach centered on \textbf{Dynamic Operator Modulation}. Instead of modulating external activations or inputs, \marksystem{} projects context vectors into bounded, multi-dimensional shifts that directly alter the internal dynamics ($A, B, C, D, \Delta$) of a frozen base model \cite{gu2022efficientlymodelinglongsequences}. Instantiated on Hydra BERT (a bi-directional Mamba extension parameterized as a quasi-separable matrix mixer \cite{hwang2024hydrabidirectionalstatespace}), \marksystem{} makes pre-trained matrices depend on diffusion timesteps, enabling dynamic, timestep-specific temporal properties.

By using bounded continuous modulations via low-rank geometries, \marksystem{} provides a dynamic, parameter-efficient alternative to static low-rank adaptation methods \cite{hu2022lora} in Hydra. We present three kernel geometries: Hypernet, Chebyshev Polynomial, and Discrete Cosine Transform (DCT) \cite{ahmed1974discrete}, mapping conditioning contexts to scalar shifts and low-rank scales. Bounding parameter modulations via hyperbolic tangent saturation and softplus constraints yields Affine Quadratic Stability (AQS) certificates for the modulated recurrence \cite{cox2018lpv}. Applied to a $23$-layer Hydra SSM, our adapters condition the frozen 111M base model using 6.3--11M trainable parameters, converting expensive pre-training into targeted fine-tuning \cite{hwang2024hydrabidirectionalstatespace}.

We evaluate \marksystem{} within an iterative mask-based diffusion framework using stability certificates, synthetic operator recovery, Markov diagnostics, and ablations. Training relies on a Context-Adaptive Token-Level Noise Rescheduling (CART) objective to provide token-level, context-aware signals during refinement \cite{ye2025dream7bdiffusionlarge}.

\paragraph{Contributions.}
\begin{enumerate}[label=(\roman*)]
    \item \textbf{LPV-SSM Operator Conditioning:} We frame conditional generation as LPV control over SSM operators, directly modulating $A$, $B$, $C$, $D$, and $\Delta$ beyond input injection and AdaLN.
    \item \textbf{\marksystem{} Kernel Family:}  We design three structural adapter geometries (Hypernet, Chebyshev, and DCT) that map conditioning vectors to low-rank modulations.
    \item \textbf{Stability Certificates and Operator Recovery:} We use bounded parameterization to obtain analytic AQS certificates and study Markov-operator recovery under matched synthetic assumptions.
    \item \textbf{Markov Operator Adaptation:} We show MaRK conditions the Markov parameter sequence, whose timestep-specific lag-norm profiles reveal distinct temporal signatures.
\end{enumerate}

\section{Linear Parameter Varying (LPV) Foundation}
\label{sec:lpv_foundation}

We adapt the base time-invariant State Space Model using the theory of Linear Parameter Varying (LPV) models \cite{cox2018lpv, shamma1988analysis}. Therefore, we view each modulated SSM layer as a continuous-time state-space system:
\begin{align}
\dot{x}(\tau) &= A(\cvector)\, x(\tau) + B(\cvector)\, u(\tau) \\
y(\tau) &= C(\cvector)\, x(\tau) + D(\cvector)\, u(\tau)
\end{align}
\noindent where $x(\tau) \in \mathbb{R}^n$ is the state, $u(\tau) \in \mathbb{R}^m$ the input, $y(\tau) \in \mathbb{R}^m$ the output, $n$ is the state dimension, $m$ the embedding dimension, $\cvector$ is a dynamic conditioning vector tied to the diffusion timestep $t$ obtained using standard DDPM/DiT sinusoidal (Fourier) encoding and two-layer MLP \cite{ho2020denoisingdiffusionprobabilisticmodels, peebles2023scalablediffusionmodelstransformers}, $\dot{x}(\tau)$ is the derivative of the state $x(\tau)$ with respect to continuous internal-time $\tau$ (i.e. $\frac{dx}{d\tau}$).

\marksystem{} modulates the continuous-time parameters $Q\in\{A, B, C, D, \Delta\}$ using learned mapping functions $\psi_\text{scale}(\cdot)$ and $\psi_\text{shift}(\cdot)$. We use the primed notation to denote the modulated parameters $Q'\in\{A', B',C', D', \Delta'\}$. 

For diagonalized scalar parameters $\Theta\in\{A_\text{log}, \Delta_\text{bias}, D\}$ we apply bounded shifts to maintain stability, where $\beta_{\Theta}$ denotes the scalar magnitude of the bounded shifts:
\begin{align}
    \Theta' &= \Theta + \beta_{\Theta} \tanh\big(\psi_\text{shift}(\cvector)\big) \label{eq:shift}
\end{align}
For dense projection matrices $\Phi\in\{B,C\}$ we apply bounded scales and shifts, where $\alpha_{\Phi}$, $\beta_{\Phi}$ denote the scalar magnitude of the bounded scales and shifts respectively. To avoid the curse of dimensionality for large matrices $B, C \in \mathbb{R}^{d}$ 
(where $d = n_{\text{groups}} \times d_{\text{state}}$; $n_{\text{groups}}$ is the number of SSM groups and $d_{\text{state}}$ is the state dimension), modulation is applied via a rank-$r$ factorized projection to these matrices:
\begin{align}
    \Phi' &= \Phi \odot \big[1 +\alpha_{\Phi} \tanh\big(\psi_\text{scale}(\cvector)\big)\big] + \beta_{\Phi}\tanh\big(\psi_\text{shift}(\cvector)\big) \label{eq:scale}
\end{align}

\textbf{Remark (Operator vs. Activation Control).} While \S\ref{sec:introduction} contrasts these paradigms conceptually, the structural distinction is formalized in Eq. (\ref{eq:shift})--(\ref{eq:scale}). By modulating the infinitesimal generator $A$ and discretization $\Delta$ rather than applying scalar affine transforms (AdaLN), MaRK directly parameterizes the system's eigenvalue distribution and frequency response. This makes the context $\cvector$ explicit in the temporal kernel's decay properties, rather than treating it only as an external signal bias.

Finally, we obtain the LPV-SSM form from equations (1) and (2):
\begin{align}
    \dot{x}(\tau) &= \bar{A'}(\cvector)\, x(\tau) + \bar{B'}(\cvector)\, u(\tau) \\
    y(\tau)   &= C'(\cvector)\, x(\tau) + D'(\cvector)\, u(\tau)
\end{align}
\noindent where $\bar{A'}$ and $\bar{B'}$ are the discretized versions of the continuous-time parameters $A'$ and $B'$, obtained via the Zero-Order Hold (ZOH) \cite{gu2022efficientlymodelinglongsequences}.

The complete input-output behavior of this LPV-SSM is fully characterized by the Markov parameter, where $\mathcal{H}(\cvector)$ is the Markov Parameter Sequence, $h_k(\tau)$ is the $k^{th}$ Markov Parameter, and $k$ is the shift index \cite{kailath1980linear}:
\begin{align}
    \mathcal{H}(\cvector) = \left\{ h_k(\tau) = C'(\cvector)\,\bar{A'^{\,k}}(\cvector)\,\bar{B'}(\cvector) \right\}_{k=0}^{\infty}
\end{align}
This object is the natural invariant of the system: two SSMs with different internal coordinates but identical Markov sequences are behaviorally indistinguishable. Crucially, by making each matrix in this product condition-dependent, MaRK adapts the full temporal operator at each generation step rather than acting externally on activations or inputs. We term this approach Markov-adapted Recurrent Kernels (\marksystem{}).

\subsection{Constructive Stability via Log-Space Parameterization.}\label{subsec:modulation}

The bounded modulations above support a constructive stability argument through the ZOH (Zero-Order Hold) discretization chain. 
In Hydra's diagonal SSM, each recurrence eigenvalue $\bar{A}_i$ is produced by the composition:
\begin{align}
    A &= -\exp(A_{\log}') < 0 \label{eq:neg_exp} \\
    \Delta &= \text{softplus}(\Delta_{\text{input}} + \Delta_{\text{bias}}') > 0 \label{eq:dt_softplus} \\
    \bar{A}_i &= \exp(A_i \cdot \Delta_i) \in (0, 1) \label{eq:zoh_unit}
\end{align}
Equation~\eqref{eq:neg_exp} follows from $\exp(x)>0$ for all $x\in\mathbb{R}$, hence $-\exp(A_{\log}')<0$. Equation~\eqref{eq:dt_softplus} holds because $\mathrm{softplus}(x)=\log(1+e^x)>0$ for all $x\in\mathbb{R}$. Finally, Equation~\eqref{eq:zoh_unit} follows since $A_i<0$ and $\Delta_i>0$ imply $A_i \cdot \Delta_i<0$, and therefore $\bar{A}_i=\exp(A_i \cdot \Delta_i)\in(0,1)$.

\begin{proposition}[Constructive AQS stability]\label{prop:constructive_aqs}
\textit{For the MaRK-modulated diagonal SSM, the identity matrix $P = I$ is a valid common Lyapunov function satisfying the AQS condition $\bar{A}(\cvector)^\top P\, \bar{A}(\cvector) - P \preceq -\varepsilon I$ for all admissible conditioning contexts $\cvector$, with stability margin $\varepsilon = 1 - \max_i \bar{A}_i^2 > 0$.}
\end{proposition}

\begin{proof}[Proof sketch]
Since $A$ is diagonal, the Lyapunov condition reduces element-wise to $\bar{A}_i^2 - 1 < 0$ for each head $i$. 
By Equations~\eqref{eq:neg_exp}--\eqref{eq:zoh_unit}, every element $\bar{A}_i \in (0,1)$ for admissible conditioning contexts. 
Because $\bar{A}_i(\delta)$ is monotonically decreasing in the shift $\delta$ (larger shift $\to$ more negative $A$ 
$\to$ more decay), the worst-case spectral radius occurs at $\delta = -\beta_A$, bounded away from $1$ by the finite 
hyper-rectangle of modulations.
\end{proof}

\section{Method}
\label{sec:method}

\subsection{Pretrained Hydra BERT SSM backbone}
\label{subsec:hydra}
We instantiate \marksystem{} on top of a pre-trained Hydra BERT backbone, a bidirectional extension of Mamba implemented through a quasiseparable matrix mixer \cite{hwang2024hydrabidirectionalstatespace}. Hydra adopts the ``matrix mixer'' view of sequence mixing, in which a sequence mixer is a structured linear map $Y = MX$ on the token sequence. In contrast to ad-hoc bidirectional SSM heuristics that combine separate forward/backward models by addition or concatenation \cite{visionmamba, maskmamba}, Hydra parameterizes a single mixer matrix whose strictly lower and upper triangles correspond to forward and backward semiseparable (SSM) structure, while the diagonal remains free and acts as a learned residual pathway. Concretely, Hydra's quasiseparable mixer admits the form (cf.\ Eq.~(3) in \cite{hwang2024hydrabidirectionalstatespace})
\begin{equation}
m_{ij}=
\begin{cases}
\overrightarrow{\mathbf{c}}_{i}^{\top}\,\overrightarrow{\mathbf{A}}_{i:j}^{\times}\,\overrightarrow{\mathbf{b}}_{j} & i>j \\
\delta_i & i=j \\
\overleftarrow{\mathbf{c}}_{j}^{\top}\,\overleftarrow{\mathbf{A}}_{j:i}^{\times}\,\overleftarrow{\mathbf{b}}_{i} & i<j
\end{cases}
\end{equation}
where $\delta_i$ is a scalar, $b_i, c_i \in \mathbb{R}^{N \times 1}$, and $A_i \in \mathbb{R}^{N \times N}$. Since this matrix features non-zero upper triangle components, it enables bidirectional processing \cite{hwang2024hydrabidirectionalstatespace}.

In our experiments, we use the released Hydra-BERT checkpoint at BERT scale (111M parameters, 23 layers) \cite{hwang2024hydrabidirectionalstatespace}. Hydra is a natural backbone for our setting because discrete diffusion language modeling is inherently bidirectional (tokens are refined using both left and right context), yet our objective requires timestep-dependent control over the internal temporal operator. \marksystem{} targets exactly the parameters that define Hydra's SSM dynamics (the continuous-time recurrence and discretization operators: $A,B,C,D,\text{ and }\Delta$ underlying the quasi-separable factors) and makes them functions of the diffusion context $\cvector$. By doing so, we demonstrate the ability of MaRK to also function as a fine-tuning method to convert an existing bidirectional objective into a context-aware diffusion objective. Throughout, we treat the Hydra backbone as frozen and learn only the \marksystem{} adapters, isolating the effect of dynamic operator modulation from changes in the base representation.

\subsection{Discrete Diffusion Language Modeling}
\label{subsec:diffusion}
We evaluate \marksystem{} in the masked discrete diffusion framework used by recent diffusion LLMs \cite{ye2025dream7bdiffusionlarge,nie2025largelanguagediffusionmodels, austin2021d3pm}. Let $\mathbf{x}_0 \in \mathcal{V}^{L}$ be a token sequence of length $L$ from the data distribution, and let $\mathcal{M}$ denote an absorbing noise token. The forward process independently corrupts tokens according to a timestep $t$ by masking each position with probability $t$ (equivalently, keeping it unmasked with probability $\alpha(t)=1-t$) \cite{nie2025largelanguagediffusionmodels}:
\begin{equation}
q(\mathbf{x}_t \mid \mathbf{x}_0) = \prod_{i=1}^{L} q(x_{t,i} \mid x_{0,i}), \qquad
q(x_{t,i}=\mathcal{M} \mid x_{0,i}) = t,\;\; q(x_{t,i}=x_{0,i}\mid x_{0,i}) = 1-t.
\label{eq:forward_mask}
\end{equation}
The reverse model is a mask predictor, $p_\theta(\mathbf{x}_0 \mid \mathbf{x}_t, t)$, that predicts all masked positions in parallel using full bidirectional context. Concretely, we parameterize $p_\theta$ with the Hydra backbone, and provide the diffusion context (timestep embedding) as the conditioning vector $\cvector$ that drives \marksystem{} operator modulation that conditions $p_\theta$ based on the diffusion context.

Following \cite{ye2025dream7bdiffusionlarge,nie2025largelanguagediffusionmodels}, we train by sampling a clean sequence $\mathbf{x}_0$, sampling $t \sim \mathcal{U}(0,1)$, then drawing  $\mathbf{x}_t \sim q(\mathbf{x}_t\mid \mathbf{x}_0)$, and minimizing the negative log-likelihood on only the masked tokens:
\begin{equation}
\mathcal{L}_{\text{diff}}(\theta)
=
\mathbb{E}_{\mathbf{x}_0,\,t,\,\mathbf{x}_t}
\Big[
w(t)\sum_{i=1}^{L} \mathbf{1}[x_{t,i}=\mathcal{M}]\; \big(-\log p_\theta(x_{0,i}\mid \mathbf{x}_t, t)\big)
\Big],
\label{eq:masked_diff_loss}
\end{equation}
where $w(t)$ is a timestep-dependent reweighting term derived from the diffusion ELBO for particular schedules \cite{nie2025largelanguagediffusionmodels}, and can be further adapted at the token level based on local context informativeness \cite{ye2025dream7bdiffusionlarge}. In our implementation, this context-adaptive reweighting is instantiated by CART (\S\ref{subsec:cart}), which emphasizes masked tokens that are surrounded by more usable (unmasked) context.

\subsection{Training Objective}
\label{subsec:cart}

To optimize training, we utilize \textbf{Context-Adaptive Token-level Noise Rescheduling (CART)}, following the token-level rescheduling paradigm introduced in Dream \cite{ye2025dream7bdiffusionlarge}. Rather than using a single global timestep-dependent weight $w(t)$ for all masked tokens at diffusion time $t$ (Eq.~\eqref{eq:masked_diff_loss}), CART applies a per-token weight $w_i(t,\mathbf{x}_t)$ that depends on how much usable (unmasked) context surrounds token $i$ in the partially-masked sequence $\mathbf{x}_t$. This matters because two masked tokens at the same $t$ can have very different predictability: a token with many nearby visible neighbors is effectively less ambiguous and should contribute a stronger learning signal, while a token whose neighborhood is mostly masked is inherently underdetermined and over-penalizing it injects gradient noise and encourages collapse toward ``easy'' tokens. CART operationalizes ``how much context'' as a distance-weighted sum of the visible tokens: each masked position receives one symmetric geometric kernel centered on every unmasked position, so nearer visible tokens contribute more and the weight grows with the amount of nearby visible context \cite{ye2025dream7bdiffusionlarge}. The resulting token-wise weights multiply the masked-token cross-entropy terms before averaging. We provide the full mathematical definition and its relationship to Eq.~\eqref{eq:masked_diff_loss} in Appendix~\ref{app:cart_details}.

\subsection{\marksystem{} Kernel Adapters}
\label{subsec:kernel_adapters}

\marksystem{} proposes the mapping $\psi(\cdot)$ that serves as the system's scheduling function, translating high-dimensional semantic context into the low-dimensional manifold of stable dynamical parameters. By varying the basis of this mapping, we can control the spectral regularity of the adaptation process. The certified setting uses the bounded modulation structure from equations~\eqref{eq:shift} and \eqref{eq:scale}, and the three adapter geometries are ablated head-to-head in \S\ref{subsec:ablations}.

\textbf{Hypernet Adapter:} The Hypernet maps the hidden state $h$ (obtained from the $\cvector$ via a mapping) directly to parameter shifts using linear projections $W$, to get the left and right low-rank matrices $U$ and $V$, functioning as the most 
parameter-efficient baseline ($\approx$ 6.3M parameters):
\begin{align}
    \Theta_{\text{shift}} &= W_{\Theta} h \quad \\
    UV_{\Phi_{\text{scale/shift}}} &= W_{\Phi_{\text{scale/shift}}} h \quad
\end{align}
The bounded scales control the gain amplification avoiding explosive states, whereas the shifts are kept optionally left unconstrained for the Hypernet kernel only, to be treated as a residual correction and act as the most expressive baseline.  

\textbf{Chebyshev Polynomial Adapter:} To enforce smoother gradients, this kernel uses orthogonal polynomial basis functions evaluated at a latent coordinate 
$z = W_z h$. The Chebyshev polynomials $T_k(z)$ (up to degree $K$) determine the continuous shifts ($\approx$ 11M parameters):
\begin{equation}
    \Theta_{\text{shift}} = \sum_{k=1}^K [W_{\Theta}(h)]_k \cdot T_k(z)
\end{equation}
For matrix elements, we utilize the basis mean $\bar{T}_k = \frac{1}{n_{\text{heads}}} \sum_i T_k(z_i)$:
\begin{equation}
    UV_{\Phi_\text{scale/shift}} = W_{\text{base}}(h) + \sum_{k=1}^K [W_{\text{basis}}(h)]_k \cdot \bar{T}_k
\end{equation}

\textbf{Discrete Cosine Transform (DCT) Kernel:} The DCT kernel captures spectral dynamics across diffusion steps, it first forms a length-$L$ signal $u$ by combining frequency-damped spectral weights from $h$  (with learnable decay exponent $\alpha$ via $1/(f+1)^{\text{softplus}(\alpha)+\epsilon}$ where $\epsilon \in (0,1)$ is a constant added for numerical stability) with a fixed orthonormal Type-II DCT cosine basis $\{\phi_f\}$, where $f$ indexes cosine frequency bins \cite{ahmed1974discrete}, then applies the same bounded structure as \eqref{eq:shift}--\eqref{eq:scale} with readouts from $u$, and for the $B,C$ base path, from $h$ as well ($\approx$ 7.1M parameters):
\begin{equation}
    u = \sum_{f} \frac{[W_{\text{spec}} h]_f}{(f+1)^{\text{softplus}(\alpha) + \epsilon}} \,\phi_f . \quad \text{where $\epsilon = 1 \times 10^{-4}$}
\end{equation}
Diagonal shifts and low-rank $U,V$ factors are
\begin{align}
    \Theta_{\text{shift}} &= W_{\Theta}u, \\
    UV_{\Phi_{\text{scale/shift}}} &= W_{\Phi_{\text{base}}}h + W_{\Phi_{\text{basis}}}u,
\end{align}

\section{Experiments and Analysis}
\label{sec:experiments}

\subsection{AQS Certificate and Diagnostics}
\label{subsec:aqs}

The Affine Quadratic Stability (AQS) certificate provides an analytic stability check for the bounded diagonal recurrence induced by 
the \marksystem{} kernels \cite{cox2018lpv}. Rather than solving a generic Semi-Definite Program (SDP) to search for a 
common Lyapunov function, we exploit the constructive parameterization chain from \S\ref{subsec:modulation} to obtain an 
analytical certificate with $P = I$.

We load the 23-layer pretrained Hydra base model and extract each layer's diagonal recurrence parameters $A_{\log}$. Since the 
maximum modulation bound $\beta_A$ is shared across the certified bounded variants, the bounded 
hyper-rectangle $[A_{\log} - \beta_A,\; A_{\log} + \beta_A]$ is kernel-independent within this setting.

For each layer $\ell$ and head $i$, we compute the worst-case discrete eigenvalue at the vertex $\delta = -\beta_A$, where 
$\bar{A}_i$ is closest to $1$ (monotonicity of $\bar{A}_i(\delta)$ in $\delta$, see Proposition~\ref{prop:constructive_aqs}), 
using the minimum discretization step $\Delta_{\min} = 0.001$. The energy function $V(x) = \|x\|^2$ (i.e., $P = I$) satisfies:
\begin{equation}
\bar{A}^\top I\, \bar{A} - I = \text{diag}(\bar{A}_i^2 - 1) \prec 0
\end{equation}
with per-layer stability margin $\varepsilon_\ell = 1 - \max_i \bar{A}_{i,\ell}^2 > 0$. All 23 layers certify with positive 
$\varepsilon_\ell$ under the stated modulation and discretization assumptions, supporting stability of the modulated Hydra recurrence across admissible conditioning 
contexts. The per-layer spectral radii and stability margins are visualized in Figure~\ref{fig:aqs_certificate}.

\begin{figure}[H]
    \centering
    \includegraphics[width=1\linewidth]{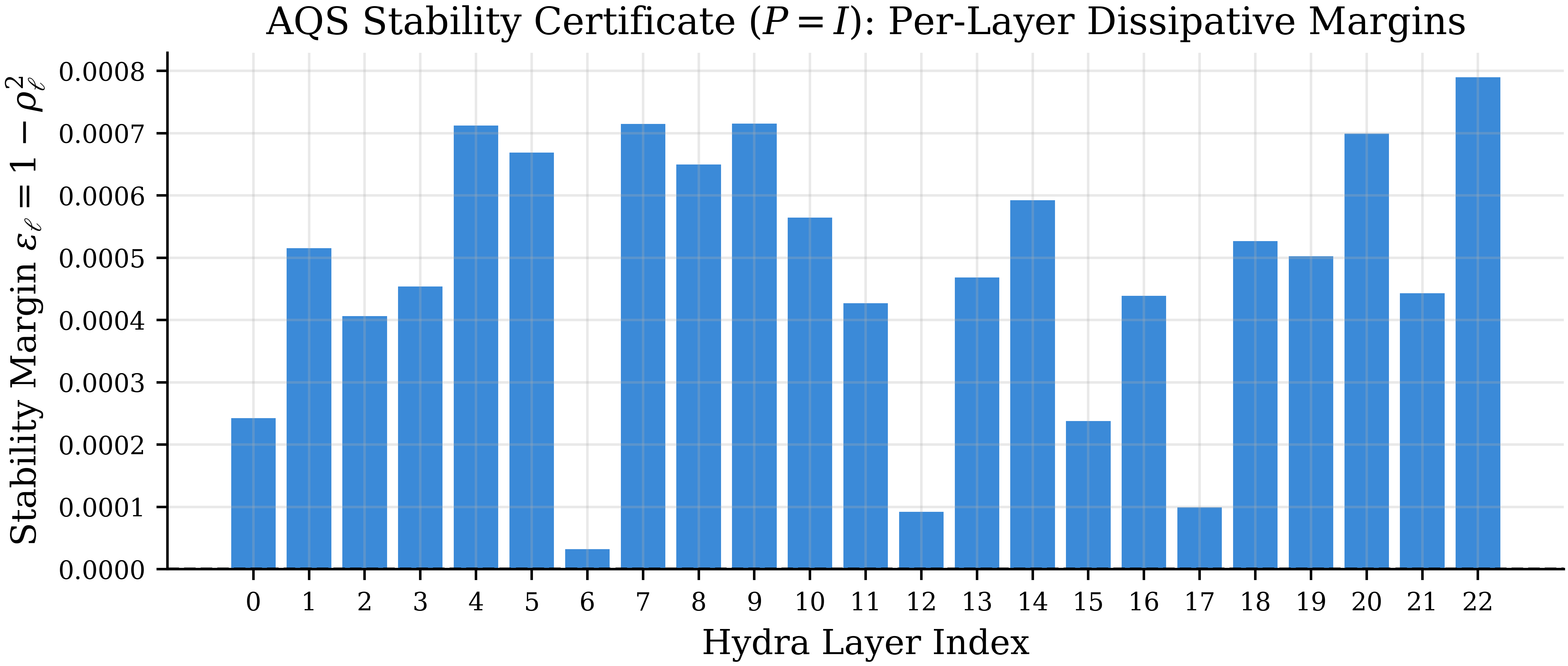}
    \caption{Stability margins ($\varepsilon_\ell$) across all the layers of the pre-trained Hydra SSM at the extreme bounds of modulation by \marksystem{}.}
    \label{fig:aqs_certificate}
\end{figure}

\subsection{Synthetic LPV Identification}
\label{subsec:identifiability}

This experiment evaluates the learnability and theoretical identifiability of the \marksystem{} architecture. While the internal state matrices inhabit an unidentifiable manifold due to non-unique factorizations, the underlying input-output dynamics are estimable under the bounded LPV assumptions below. We evaluate recovery using the coordinate-invariant Markov Operator Error (MOE): $\| \mathcal{H}_{\text{true}} - \hat{\mathcal{H}}_{\text{est}} \|_F$ (where $\mathcal{H}_{\text{true}}$ is the true Markov Parameter Sequence and $\hat{\mathcal{H}}_{\text{est}}$ is the estimated Markov Parameter Sequence by \marksystem{}). We train synthetic Hypernet, Chebyshev, and DCT adapter variants on $N \in [50,1600]$ trajectories with Gaussian output noise, keeping the backbone frozen to mirror our training regime; full synthetic protocol details are given in Appendix~\ref{app:experiment_details}.

\begin{proposition}[Operator Identifiability and Convergence]
\label{prop:identifiability}
\textit{Let $\mathcal{M}$ be an Affine Quadratically Stable LPV-SSM with strictly bounded representations, observed under standard excitation and mixing conditions. The precise statistical model, including the frozen-context identification protocol, bounded-density innovation assumptions, and operator-level persistent excitation, is specified in Appendix~\ref{app:operator_identifiability_proof}. Then (i)~the input-output operator defined by the frozen-time Markov sequence $\mathcal{H}(\cvector)=\{\bar{C'}(\cvector)\bar{A'^{\,k}}(\cvector) \bar{B'}(\cvector)\}_{k=0}^{\infty}$ is identifiable up to a structural equivalence class. (ii)~Under squared loss, the empirical estimation error converges at the optimal parametric rate $O(1/\sqrt{N})$, where $N$ is the total number of samples.}
\end{proposition}

\begin{proof}[Proof Sketch]
We define model equivalence by equality of the input-output Markov sequence, which removes the non-identifiability of internal state coordinates while preserving the observable operator. The convergence argument then follows three standard steps. (i) \textbf{Mixing:} AQS gives geometric decay of the Markov operators; with persistently exciting conditioning and sub-Gaussian inputs/noise, as used in the synthetic recovery protocol, this places the induced temporal process in the exponentially $\beta$-mixing regime \cite{boyd1994linear}. (ii) \textbf{Complexity:} The $\tanh$-bounded parameterization makes the operator class uniformly bounded and Lipschitz, so Talagrand's contraction lemma yields Rademacher complexity scaling as $O(1/\sqrt{N})$ under the mixing relaxation of i.i.d. sampling \cite{LedouxTalagrand1991}. (iii) \textbf{Generalization:} Standard squared-loss bounds for dependent sequences transfer this capacity bound to the empirical operator-estimation error \cite{bartlett2002rademacher}. Appendix~\ref{app:operator_identifiability_proof} gives the full proof, including the structural equivalence relation and the cited statistical premises.
\end{proof}

\begin{figure}[H]
  \centering
  \begin{subfigure}[b]{0.48\textwidth}
    \centering
    \includegraphics[width=\textwidth]{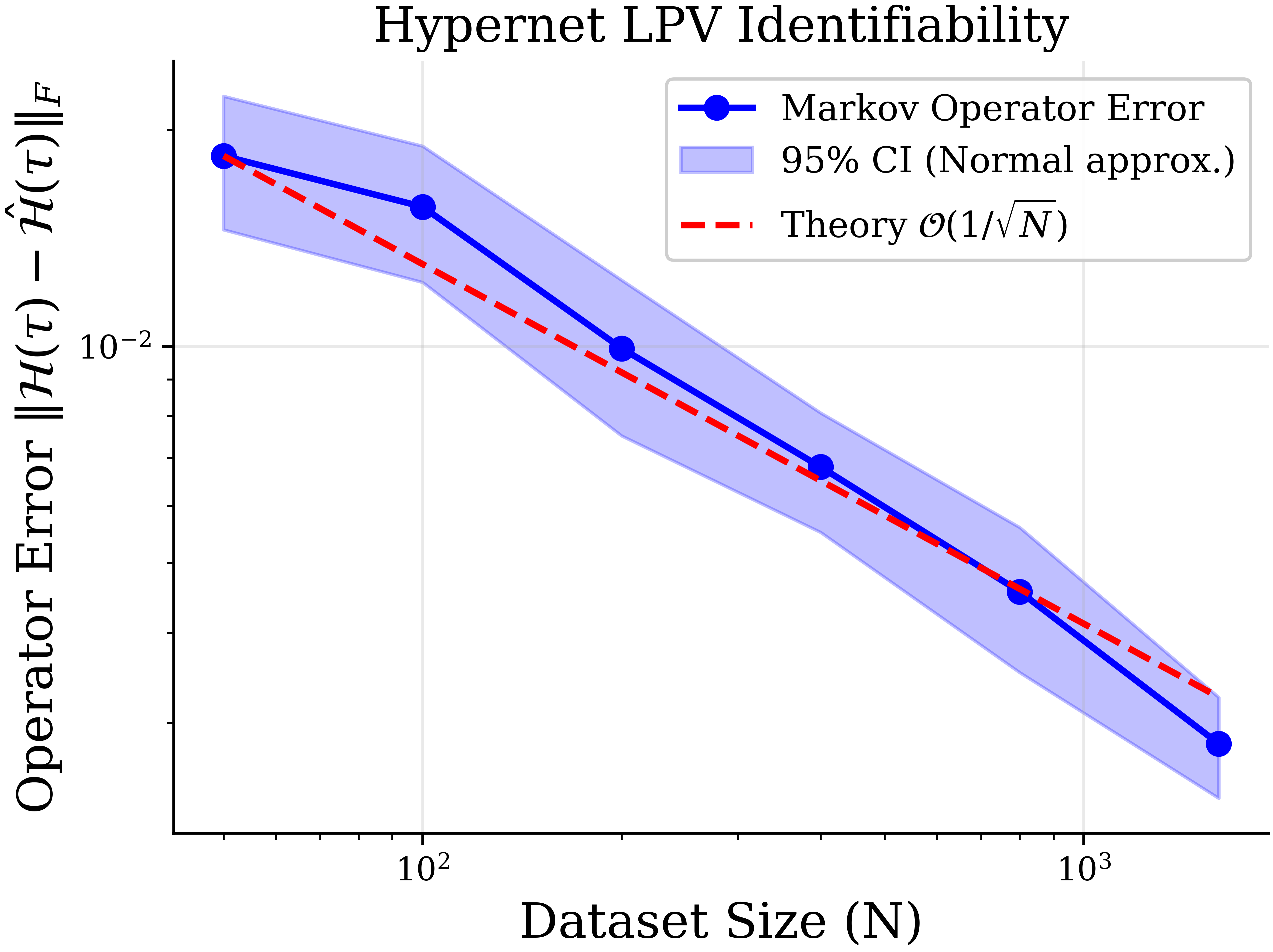}
    \caption{}
    \label{fig:moe_hypernet}
  \end{subfigure}
  \hfill
  \begin{subfigure}[b]{0.48\textwidth}
    \centering
    \includegraphics[width=\textwidth]{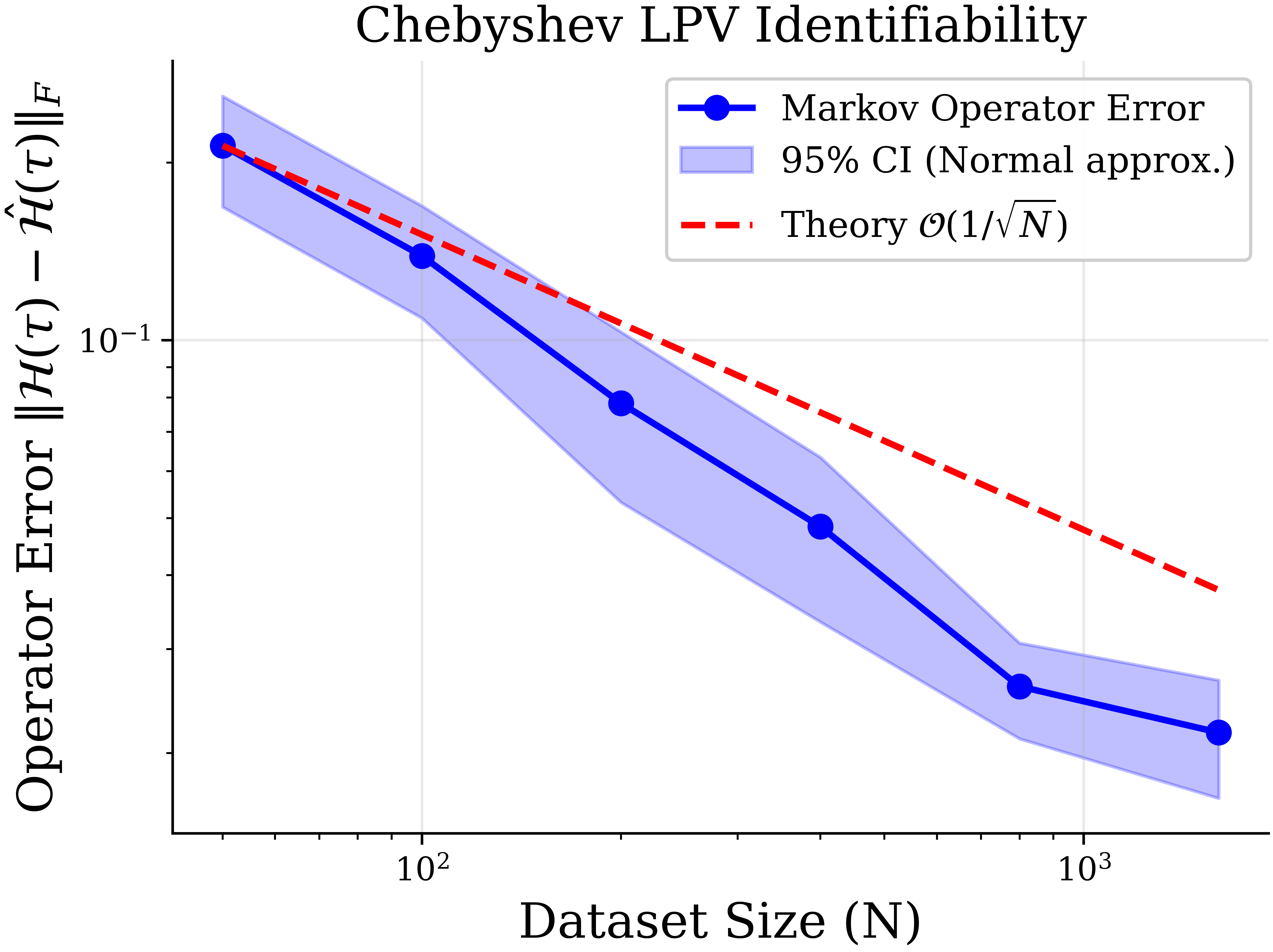}
    \caption{}
    \label{fig:moe_chebyshev}
  \end{subfigure}
  \hfill
  \\
  \begin{subfigure}[b]{0.48\textwidth}
    \centering
    \includegraphics[width=\textwidth]{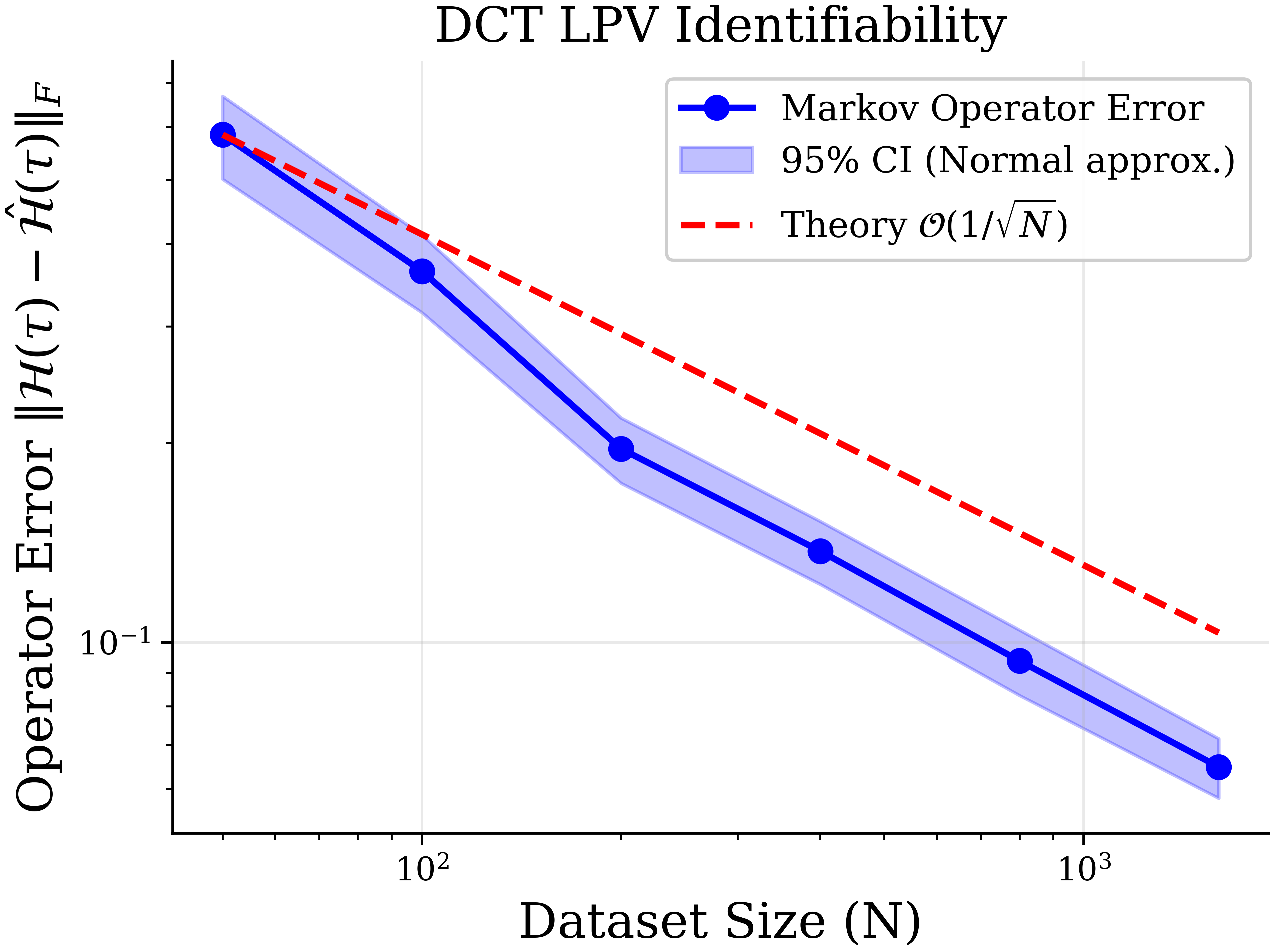}
    \caption{}
    \label{fig:moe_dct}
  \end{subfigure}
  \caption{Markov Operator Error vs Dataset Size for synthetic Hypernet, Chebyshev, and DCT LPV recovery. Solid lines denote 15-seed means with normal-approximation 95\% confidence intervals (shaded). Red dashed lines show an $O(N^{-1/2})$ reference normalized at the smallest dataset size.}
  \label{fig:moe_identifiability}
\end{figure}

% \begin{figure}[H]
%     \centering
%     \begin{subfigure}[b]{0.48\textwidth}
%         \centering
%         \includegraphics[width=\textwidth]{./chebyshev_synthetic_recovery.png}
%         \caption{Chebyshev LPV Identifiability}
%         \label{fig:moe_chebyshev}
%     \end{subfigure}
%     \hfill
%     \begin{subfigure}[b]{0.48\textwidth}
%         \centering
%         \includegraphics[width=\textwidth]{./dct_synthetic_recovery.png}
%         \caption{DCT LPV Identifiability}
%         \label{fig:moe_dct}
%     \end{subfigure}
%     \caption{Markov Operator Error vs Dataset Size. Solid lines denote 15-seed bootstrap means with 95\% confidence intervals (shaded). Red dashed lines indicate the theoretical $O(1/\sqrt{N})$ bound.}
%     \label{fig:moe_convergence}
% \end{figure}
% \begin{figure}[H]
%         \centering
%         \includegraphics[width=0.48\textwidth]{./hypernet_synthetic_recovery.png}
%         \caption{Hypernet LPV Identifiability}
%         \label{fig:moe_hypernet}
% \end{figure}

\textbf{Analysis of Convergence.} Figure \ref{fig:moe_identifiability} illustrates the log-log scaling of the MOE. All three synthetic kernels exhibit systematic operator-error decay as the number of trajectories increases, supporting the predicted $O(1/\sqrt{N})$ estimation trend in Proposition 4.1 under matched LPV recovery conditions. Fitting the endpoints over $N=50$ to $1600$ gives approximate empirical rates of $N^{-0.54}$ for Hypernet, $N^{-0.66}$ for Chebyshev, and $N^{-0.64}$ for DCT, consistent with parametric convergence up to kernel-dependent constants. Chebyshev achieves the lowest structured-basis error after the smallest sample regime, matching its smooth polynomial bias. DCT has a larger absolute constant, plausibly reflecting the additional bandlimited spectral restriction, but its monotone decay indicates that this restriction does not destroy operator-level recoverability. The experiment therefore supports the ability of \marksystem{} kernel geometries to recover the coordinate-invariant Markov operator in a well-specified synthetic setting, while the theorem supplies a conditional statistical guarantee.

\subsection{Dynamics Analysis}
\label{subsec:dynamics}

To test whether \marksystem{} conditions the temporal operator, we measure the Markov parameter sequence induced by each frozen diffusion timestep. For each adapter family, we evaluate the head-averaged Markov norm $\|h_k(t)\|_2$ over lags $k \leq 4096$ at $t \in \{0, 0.25, 0.5, 0.75, 1.0\}$, and compare it with the unmodulated Hydra backbone. Since $\mathcal{H}(t)=\{C(t)A^k(t)B(t)\}_{k\geq 0}$ is the input-output memory profile of the SSM, separation between these curves provides evidence of timestep-dependent temporal kernels rather than merely timestep-dependent activations. The experiment therefore asks two questions: whether \marksystem{} can reshape the Markov sequence as the diffusion process changes, and whether the resulting kernels remain dissipative at long lags.

\begin{figure}[H]
    \centering
    \includegraphics[width=1\linewidth]{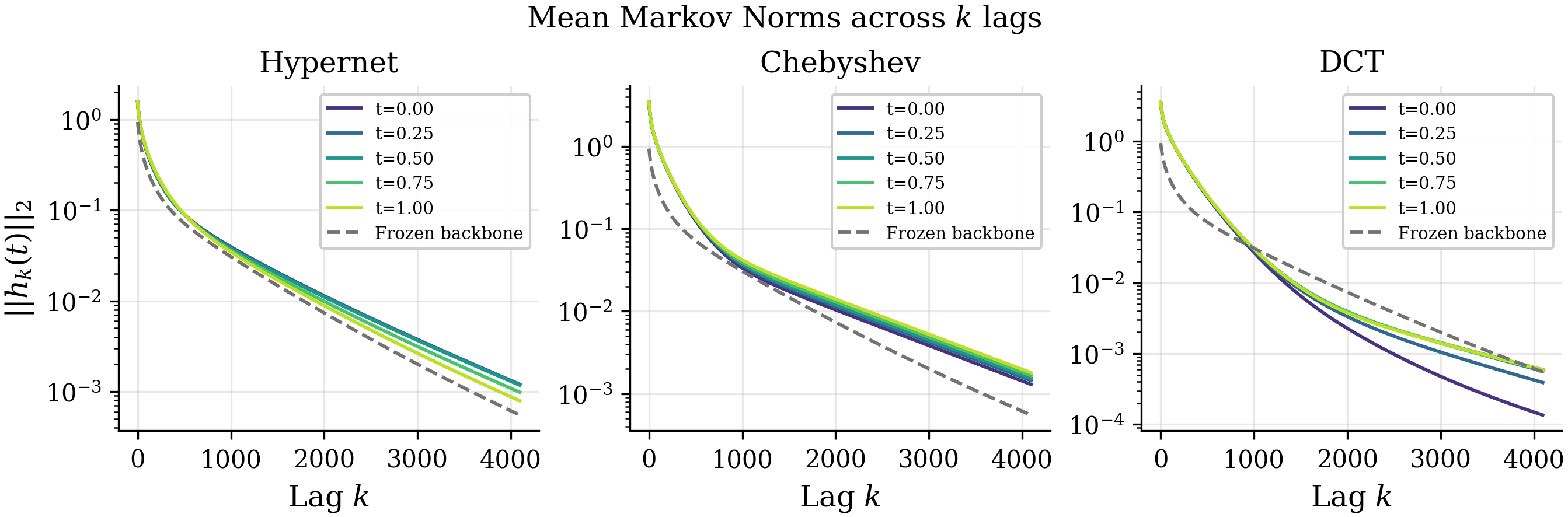}
    \caption{Head-resolved mean Markov parameter norm versus lag for Hypernet, Chebyshev, and DCT \marksystem{} adapters at different diffusion timesteps $t$. The dashed curve denotes the frozen Hydra backbone without timestep-conditioned operator modulation.}
    \label{fig:markov_norm_lag}
\end{figure}

\textbf{Analysis of Markov Dynamics.} Figure \ref{fig:markov_norm_lag} shows that all three adapter geometries preserve a stable decay envelope while producing distinct timestep-conditioned Markov sequences. The curves remain approximately geometric over thousands of lags, consistent with the AQS construction, but they do not collapse onto the frozen backbone: \marksystem{} changes how much input-output energy is retained at each temporal scale. Hypernet produces the broadest upward shift relative to the backbone, especially at medium and long lags, suggesting an expressive increase in memory persistence. Chebyshev yields a tightly ordered family of smooth decay curves, matching the low-oscillation bias of polynomial modulation and indicating controlled, stable variation across diffusion time. DCT exhibits the most selective redistribution of energy: high-noise timesteps amplify early-lag response, while lower timesteps attenuate more rapidly and cross the backbone at intermediate lags, consistent with spectral shaping rather than uniform gain scaling. Together, these signatures support the claim that \marksystem{} conditions the Markov parameter sequence, giving the diffusion timestep explicit influence over the SSM memory kernel while preserving long-horizon dissipation.

\textbf{Remark on Operator Adaptation.} Figure \ref{fig:markov_norm_lag} also illustrates all variants of \marksystem{} possess a distinct signature compared to the frozen Hydra backbone. Since the Markov parameter sequence $\{\mathcal{H}_k\}_{k=0}^{\infty}$ is the core characteristic of the SSM, it serves as a canonical operator for the system's dynamics. This is because the sequence uniquely determines the transfer function and the impulse response of the model \cite{kailath1980linear}, remaining invariant under any similarity transformation of the state-space coordinates \cite{chen1998linear}. Consequently, the observed divergence in the norm signature across $k$-lags is consistent with a shift in the model's impulse response. This shift suggests that \marksystem{} does not merely perform a linear re-projection of features, but also adapts the underlying dynamical operators (specifically the spectral properties and memory decay) to align with the temporal dependencies of the target task. This supports viewing our approach as a principled mechanism for operator-level fine-tuning.

\subsection{Ablations}\label{subsec:ablations}

We ablate the three \marksystem{} adapter geometries under the same frozen Hydra backbone, diffusion training objective, and CART-weighted validation protocol. The goal is to isolate the effect of the operator-modulation map $\psi(\cdot)$: direct Hypernet projections, smooth Chebyshev polynomial bases, or bandlimited DCT bases. Table~\ref{tab:ablation} reports $\mathcal{L}_\text{CART}$, where lower values indicate better reconstruction of masked tokens under the same context-adaptive weighting used for training. We include trainable adapter size to separate geometry effects from parameter count, and interpret the results as an internal \marksystem{} design ablation rather than as a full scaling comparison against separately trained diffusion language models.

\begin{table}[H]
\centering
\caption{CART losses ($\mathcal{L}_\text{CART}\downarrow$) across 10 seeds (mean $\pm$ 95\% CI, $1.96\times$SEM) and trainable adapter sizes across completed validation datasets. \textbf{Bold} = Best. \underline{Underline} = Second Best.}
\label{tab:ablation}
\resizebox{\textwidth}{!}{ % Resizes table to fit page width
\begin{tabular}{@{}lccccccc@{}} % Changed to 9 columns (l + 8c)
\toprule
\textbf{Kernel} & \textbf{Params} & \textbf{PTB} & \textbf{WikiText} & \textbf{Lambada} & \textbf{AG News} & \textbf{Arxiv} & \textbf{Avg.} \\ \midrule
Hypernet  & 6.3M & 3.39 ± 0.31 & 3.73 ± 0.23 & 3.70 ± 0.17 & 3.91 ± 0.20  & 4.12 ± 0.10 & 3.77 \\[1ex]

Chebyshev & 11.0M & \textbf{2.42 ± 0.26} & \underline{2.55 ± 0.16} & \textbf{2.40 ± 0.10} & \textbf{2.85 ± 0.14} & \textbf{2.55 ± 0.05} & \textbf{2.55} \\[1ex]

DCT       & 7.1M & \underline{2.48 ± 0.27} & \textbf{2.55 ± 0.14} & \underline{2.47 ± 0.11} & \underline{2.86 ± 0.14} & \underline{2.61 ± 0.07} & \underline{2.59} \\ \bottomrule
\end{tabular}
}
\end{table}

% The structured bases improve over the direct Hypernet projection on every completed validation domain. Chebyshev obtains the best average CART loss and leads on WikiText, Lambada, AG News, and Arxiv, while using nearly the same trainable parameter budget as Hypernet and substantially fewer parameters than DCT. DCT remains close on average and gives the strongest PTB score, indicating that the spectral parameterization can be competitive when the validation distribution benefits from stronger early-lag shaping. The main ablation signal is therefore not simply that more adapter parameters help: the lowest average loss comes from the smooth Chebyshev basis rather than the largest DCT adapter or the more direct Hypernet mapping.

Both structured bases improve on the direct Hypernet projection on every completed validation dataset. Chebyshev gives the lowest average loss (2.55) and leads DCT on four of the five datasets, tying it on WikiText; DCT is 0.04 nats behind on the average (2.59). The cause is structural rather than statistical. The timestep is a continuous scalar and adjacent timesteps define nearly identical denoising tasks, so the modulation map should vary smoothly, but the Hypernet MLP consumes each timestep as an independent input and must learn that smoothness from data. It can therefore fit the training timesteps accurately while varying freely between them, spending capacity on high-frequency structure the diffusion process does not need. Chebyshev and DCT encode the smoothness in the basis itself: five global polynomials, or eight frequencies with a learned spectral decay that attenuates the high end. Projecting the modulation map onto a compact basis functions set acts as a strong structural regularizer, constraining the functional space without compromising expressive capacity. Table \ref{tab:ablation} shows the effect at equal budget: Hypernet and DCT carry 6.3M and 7.1M adapter parameters and still differ by 1.18 nats, while the further 3.9M that separate DCT from Chebyshev (11.0M) buys an additional 0.04 nats. Hypernet does not overfit; it lacks the prior. Chebyshev then separates from DCT only on Arxiv (2.55 vs. 2.61), Lambada (2.40 vs. 2.47) and PTB (2.42 vs. 2.48), with WikiText and AG News inside seed noise, and DCT is the cheaper of the two at inference, so Chebyshev is the accuracy ceiling and DCT the better parameter and latency trade.

\begin{table}[H]
\centering
\caption{Parameter ablation: CART validation loss ($\mathcal{L}_\text{CART}\downarrow$) on the WikiText dataset across 10 seeds (mean $\pm$ 95\% CI, $1.96\times$SEM). "Full" modulates all five parameters; "None" modulates none (frozen Hydra). \textbf{Bold} = Best per kernel.}
\label{tab:loo}
\setlength{\tabcolsep}{2.6pt}
\resizebox{\textwidth}{!}{ % Resizes table to fit page width
\begin{tabular}{@{}lcccccccc@{}} % Changed to 9 columns (l + 8c)
\toprule
\textbf{Kernel} & \textbf{Full} & \textbf{All Except A} & \textbf{A Only} & \textbf{$\Delta$ Only} & \textbf{B,C Only} & \textbf{D Only} & \textbf{All Except $\Delta$} & \textbf{None} \\ \midrule
Hypernet  & \textbf{3.73 ± 0.26} & 3.91 ± 0.27 & 4.07 ± 0.29 & 4.04 ± 0.29 & 4.08 ± 0.29 & 3.78 ± 0.26 & 3.74 ± 0.26 & 3.99 ± 0.29 \\[1ex]

Chebyshev & \textbf{2.55 ± 0.18} & 2.78 ± 0.19 & 3.44 ± 0.22 & 3.92 ± 0.26 & 3.06 ± 0.21 & 3.90 ± 0.26 & 2.56 ± 0.18 & 4.00 ± 0.27 \\[1ex]

DCT       & \textbf{2.55 ± 0.16} & 2.80 ± 0.16 & 3.24 ± 0.17 & 3.99 ± 0.24 & 3.05 ± 0.17 & 3.90 ± 0.24 & 2.59 ± 0.16 & 4.00 ± 0.24 \\ \bottomrule
\end{tabular}
}
\end{table}

A full leave-one-out parameter ablation across all five state-space model (SSM) parameters ($A, B, C, D, \Delta$) isolated the performance contributions of each operator-modulation component. Modulating matrix $A$ proves necessary across all settings; freezing $A$ (All Except $A$) increases validation loss by +0.18 to +0.25 points. This performance gap confirms that \marksystem{} cannot be reduced to standard Mamba-style selection ($B, C, \Delta$), as the $B,C$-Only variant leaves a substantial loss degradation of +0.34 to +0.51 points. This result aligns with the functional distinction where input gating controls instantaneous throughput while $A$ governs persistent memory horizon \cite{gu2022efficientlymodelinglongsequences, gu2024mambalineartimesequencemodeling}. Furthermore, modulating $\Delta$ yields marginal independent gains due to the redundancy in spectral transformation paths: the operator dynamics can be shifted directly via $A$ or indirectly via $\Delta$ through zero-order hold discretization $\bar{A} = \exp(A_{\text{cont}} \cdot \Delta)$. The direct spectral path dominates empirically, as $A$-Only outperforms $\Delta$-Only by up to 0.75 points, and All Except $\Delta$ achieves the second-best score across every kernel. Nevertheless, single-parameter configurations trail the Full baseline by up to +0.50 points, demonstrating that optimal performance requires joint parameter modulation.

\begin{table}[H]
\centering
\caption{Inference performance and relative latency overhead over Frozen Hydra at batch size 8 on an NVIDIA A100-SXM4-80GB with the same frozen 23-layer backbone. Each variant: 15 seeds, 100 forward passes per seed, means $\pm$ 95\% CI (t-distribution, 14df). \textbf{Bold} = MaRK (Ours)}
\label{tab:latency}
\resizebox{\textwidth}{!}{%
\begin{tabular}{@{}lccccccc@{}}
\toprule
\textbf{Variant} & 
\thead{\textbf{Latency  (ms)} \\ \textbf{seq=4096}} & 
\thead{\textbf{Memory (MiB)} \\ \textbf{seq=4096}} & 
% \thead{\textbf{Throughput} \\ \textbf{(samples/s)}} & 
\thead{\textbf{Total} \\ \textbf{Params (M)}} & 
\thead{\textbf{Additional} \\ \textbf{Params (M)}} &
\thead{\textbf{Overhead} \\ \textbf{seq=512}} & 
\thead{\textbf{Overhead} \\ \textbf{seq=4096}} \\ \midrule
% \thead{\textbf{Absolute Overhead} \\ \textbf{seq=4096}}\\ \midrule
Frozen Hydra & 267.43 ± 0.06 & 4446.58 $\pm$ 0.06 & 111.68 & 0.00 & - & - \\[1ex]
Input Injection & 270.86 ± 0.03 & 4455.25 $\pm$ 0.06 & 113.96 & 2.28 & +3.6\% & +1.3\% \\[1ex]
AdaLN Zero & 276.84 ± 0.02 & 4474.52 $\pm$ 0.60 & 118.51 & 6.84 & +7.6\% & +3.5\% \\[1ex]
\textbf{Hypernet} & 304.75 ± 0.03 & 4470.62 $\pm$ 0.06 & 117.95 & 6.27 & +65.6\% & +14.0\% \\[1ex]
\textbf{Chebyshev} & 309.74 ± 0.10 & 4488.61 $\pm$ 0.06 & 122.65 & 10.97 & +104.0\% & +15.8\% \\[1ex]
\textbf{DCT}       & 307.28 ± 0.05 & 4473.93 $\pm$ 0.06 & 118.75 & 7.07 & +88.7\% & +14.9\% \\ \bottomrule
\end{tabular}%
}
\end{table}

At sequence length 4096, \marksystem{} incurs a modest 14\% to 16\% latency increase while maintaining memory overhead below 1\% relative to the frozen Hydra backbone. The per-layer conditioning computation introduces a constant temporal additive cost ($\sim$37 ms to 42 ms) that remains independent of sequence length. Consequently, while conditioning overhead dominates leading to 65\% to 104\% additional latency at length 512, this cost amortizes rapidly as sequence length scales to 4096 and beyond. Because \marksystem{} scales $O(1)$ with sequence length and $O(N)$ with diffusion sampling steps, its asymptotic runtime trajectory aligns with standard AdaLN Zero and Input-Injection baselines.

% \begin{table}[H]
% \centering
% \caption{Overhead scaling at batch size 8 on NVIDIA A100-SXM4-80GB with the same frozen 23-layer backbone. Each variant: 15 seeds, 100 forward passes per seed, means $\pm$ 95\% CI (t-distribution, 14df)}
% \label{tab:overhead}
% \resizebox{\textwidth}{!}{%
% \begin{tabular}{@{}lccccc@{}}
% \toprule
% \textbf{Variant} & 
% \thead{\textbf{Overhead at} \\ \textbf{(seq=512)}} & 
% \thead{\textbf{Overhead at} \\ \textbf{seq=4096}} & 
% \thead{\textbf{Absolute Overhead} \\ \textbf{at seq=4096 (ms)}} \\ \midrule
% Input-Injection & 270.86 ± 0.03 & 4455.25 $\pm$ 0.06 & 29.54 $\pm$ 0.00 \\[1ex]
% AdaLN-Zero & 276.84 ± 0.02 & 4474.52 $\pm$ 0.60 & 28.90 $\pm$ 0.00 \\[1ex]
% \textbf{Hypernet} (MaRK) & 304.75 ± 0.03 & 4470.62 $\pm$ 0.06 & 26.25 $\pm$ 0.00 \\[1ex]
% \textbf{Chebyshev} (MaRK) & 309.74 ± 0.10 & 4488.61 $\pm$ 0.06 & 25.83 $\pm$ 0.01 \\[1ex]
% \textbf{DCT} (MaRK)       & 307.28 ± 0.05 & 4473.93 $\pm$ 0.06 & 26.04 $\pm$ 0.00 \\ \bottomrule
% \end{tabular}%
% }
% \end{table}

\section{Conclusion}
\label{sec:conclusion}
We introduced \textbf{Dynamic Operator Modulation}, an operator-level conditioning perspective that shifts SSM adaptation from external signal injection toward direct reconfiguration of the recurrent operator. By mapping context to bounded, low-rank modulations of the frozen backbone, \marksystem{} induces a context-indexed family of stable input-output dynamics. This operator-level perspective is our central contribution: adapting an SSM for iterative tasks like diffusion can be understood as a controlled deformation of the model's underlying Markov parameters.

Our framework provides theoretical support (through Affine Quadratic Stability certificates and Rademacher complexity bounds) and empirical evidence across multiple adapter geometries. The observed shift in Markov signatures suggests that parameter-efficient fine-tuning can arise from the structure of this modulation, rather than only from an additive patch. Overall, these results suggest that the temporal logic of an SSM can be made context-aware in the studied setting while preserving stability and recoverability under the stated assumptions, opening new avenues for principled adaptation in non-stationary sequence modeling. Systematic empirical comparison against established conditioning frameworks on downstream benchmarks, and how \marksystem{} holds under scaling laws remain important future directions; we discuss these limitations further in Appendix~\ref{app:limitations}.

\section{Acknowledgements}
\label{sec:acknowledgements}
% We would like to thank the NVIDIA AI Technology Center, Hong Kong (NVAITC-HK) for providing the necessary resources needed to run the experiments. Additionally, Dr. Yuhao Wang for reviewing our work and guiding our project.
This project was made possible by the generous support provided by the NVIDIA AI Technology Center (NVAITC). Additionally, we gratefully acknowledge Dr. Yuhao Wang from the Department of Data Science at City University of Hong Kong for his assistance during the final stages of this project. Xiangyu Zhao was partially supported by National Natural Science Foundation of China (No.62502404), Hong Kong Research Grants Council (Research Impact Fund No.R1015-23, Collaborative Research Fund No.C1043-24GF, RGC Research Fellow Scheme No.RFS2627-1S03, General Research Fund No.11218325, No.11212926), Institute of Digital Medicine of City University of Hong Kong (No.9229503), Huawei (Huawei Innovation Research Program), Tencent (CCF-Tencent Open Fund, Tencent Rhino-Bird Focused Research Program, Tencent University Cooperation Project), Kuaishou (CCF-Kuaishou Large Model Explorer Fund No.2025008, Kuaishou University Cooperation Project), Alibaba (CCF-Taobao \& Tmall Group TECH Kangaroo Fund 2026001, Alibaba Innovative Research (AIR) Program), Didi (CCF-Didi Gaia Scholars Research Fund), and Bytedance.

\newpage
% \begin{thebibliography}{9}

% \bibitem{cox2018linear}
% P.~B. Cox.
% \newblock \textit{Linear Parameter-Varying Systems: Identifiability, Data-Driven Modeling and Control}.
% \newblock PhD thesis, Eindhoven University of Technology, 2018.

% \bibitem{gu2022efficientlymodelinglongsequences}
% A. Gu, K. Goel, and C. Ré.
% \newblock \textit{Efficiently Modeling Long Sequences with Structured State Spaces}. 
% \newblock arXiv preprint arXiv:2111.00396, 2022.

% \bibitem{boyd1994linear}
% S.~Boyd, L.~El Ghaoui, E.~Feron, and V.~Balakrishnan.
% \newblock \textit{Linear Matrix Inequalities in System and Control Theory}.
% \newblock SIAM, 1994.

% \bibitem{ledoux1991probability}
% M.~Ledoux and M.~Talagrand.
% \newblock \textit{Probability in Banach Spaces: Isoperimetry and Processes}.
% \newblock Springer Science \& Business Media, 1991.

% \bibitem{bartlett2002rademacher}
% P.~L. Bartlett and S.~Mendelson.
% \newblock Rademacher and Gaussian complexities: Risk bounds and structural results.
% \newblock \textit{Journal of Machine Learning Research}, 3(Nov):463--482, 2002.

% \bibitem{ye2025dream}
% J.~Ye, Z.~Xie, L.~Zheng, J.~Gao, Z.~Wu, X.~Jiang, Z.~Li, and L.~Kong.
% \newblock Dream 7B: Diffusion Large Language Models.
% \newblock \textit{arXiv preprint arXiv:2508.15487}, 2025.

% \end{thebibliography}
\bibliographystyle{unsrt}

\bibliography{references}

\newpage
\appendix

\section{Discussions}
\label{app:discussions}

\subsection{Notation Summary}\label{app:notation_summary}
\begin{table}[H]
\centering
\caption{Compact symbol guide for the main text.}
\begin{tabular}{@{}p{0.22\linewidth}p{0.70\linewidth}@{}}
\toprule
\textbf{Symbol} & \textbf{Meaning} \\ \midrule
$t,\tau$ & diffusion timestep and internal continuous time \\
$\cvector$ & context/timestep embedding used as the LPV scheduling variable \\
$Q,Q'$ & base and modulated SSM parameters, $Q\in\{A,B,C,D,\Delta\}$ \\
$A(\cvector),B(\cvector)$ & LPV-SSM matrix-valued functions for recurrence and read-in parameters\\
$C(\cvector),D(\cvector)$ & LPV-SSM matrix-valued functions for read-out and skip parameters\\
$x(\tau), u(\tau), y(\tau)$ & LPV-SSM vector-valued functions for state, input and output\\
$A,B,C,D,\Delta$ & SSM recurrence, read-in, read-out, skip, and discretization parameters \\
$A_{\log}',\Delta_{\text{input}},\Delta_{\text{bias}}'$ & modulated log-recurrence, discretization input, and modulated discretization bias \\
$\Theta,\Phi$ & diagonal parameters $\{A,\Delta,D\}$ and dense projections $\{B,C\}$ \\
$\psi_{\text{scale}},\psi_{\text{shift}}$ & adapter maps that emit bounded scales and shifts \\
$\alpha_\Phi,\beta_\Phi,\beta_\Theta$ & modulation magnitudes for scales and shifts \\
$P,I,\varepsilon$ & Lyapunov matrix, identity matrix, and AQS stability margin \\
$\bar{A},\bar{B}$ & ZOH-discretized recurrence and read-in parameters \\
$\mathcal{H}(\cvector),h_k$ & context-indexed Markov parameter sequence and its $k$th lag term \\
$\mathcal{V},\mathcal{M},\theta$ & vocabulary, absorbing mask token, and trainable model parameters \\
$h,W_\cdot,z,T_k,\phi_f$ & adapter hidden state, learned readouts, Chebyshev coordinate/polynomial, and DCT basis \\
$w(t),w_i(t,\mathbf{x}_t)$ & global diffusion weight and CART token-level weight \\
$U,V,r$ & low-rank factors and rank used for dense-parameter modulation \\
\bottomrule
\end{tabular}
\end{table}

\subsection{The \marksystem{} Kernel Classes}
We propose three principal variants of the \marksystem{} kernels: Hypernet, Chebyshev, and DCT. Each kernel structurally modifies the recurrent dynamics ($A, \Delta$) and input/output projections ($B, C$) of the base State Space Model. In the certified bounded setting, shifts are modulated by bounded activation functions (hyperbolic tangent gating) to support AQS and prevent unbounded activation growth.

The three variants should be read as different parameterizations of the same bounded operator-modulation map rather than as unrelated architectures. Hypernet places the weakest structural prior on the context-to-operator map and is therefore useful as a flexible baseline. Chebyshev and DCT instead restrict the modulation to structured bases: Chebyshev favors globally smooth variation over the conditioning coordinate, while DCT favors bandlimited spectral variation with explicit attenuation of higher-frequency components. This distinction is important because the empirical comparison in Table~\ref{tab:ablation} tests not only whether \marksystem{} can modulate the SSM operator, but also which inductive bias makes that modulation reliable under a frozen backbone and finite adaptation budget.

\begin{itemize}[leftmargin=*]
    \item \textbf{Expressivity.} Hypernet can represent the least constrained modulation among the three variants, but this flexibility may allocate capacity to changes that are not necessary for the task. Chebyshev and DCT reduce this search space by forcing modulations to pass through low-dimensional basis expansions.
    \item \textbf{Optimization and parameter efficiency.} The basis kernels trade some pointwise flexibility for smoother parameter trajectories. This can make the adapter easier to regularize and interpret, although the actual compute and parameter cost still depends on basis degree, frequency budget, and low-rank factor sizes.
    \item \textbf{Interpretability.} Chebyshev coefficients can be interpreted as polynomial components of the operator path, while DCT coefficients separate low- and high-frequency modulation content. Hypernet is less directly interpretable, since the MLP output mixes these effects before the bounded shifts are applied.
\end{itemize}

\subsubsection{Hypernet Adapter}
The Hypernet acts as the fundamental \marksystem{} baseline, employing standard linear projections directly from the hidden conditioning variable to the factorized shift components. This variant serves as the most expressive baseline because the MLP can learn a comparatively unconstrained mapping from the conditioning vector to the modulations. Its main role in the study is therefore diagnostic: it asks whether direct learned projections are sufficient, before imposing the smoother basis restrictions used by Chebyshev or DCT.

\begin{algorithm}[H]
\caption{Hypernet \marksystem{} Kernel}
\label{alg:hypernet_kernel}
\begin{algorithmic}[1]
\State \textbf{Input:} Context vector $\cvector$, Base parameters $A_{\log}, B, C, \Delta, D$
\State \textbf{Hyperparameters:} Factorization constraints $\alpha, \beta$
\State $h \gets \text{MLP}(\cvector)$ \Comment{Extract hidden representation}
\State $A_{\text{shift}}, \Delta_{\text{shift}}, D_{\text{shift}} \gets \text{Linear}_{A, \Delta, D}(h)$
\State $B_{\text{scale}}, B_{\text{shift}} \gets \text{LowRankProject}(h)$
\State $C_{\text{scale}}, C_{\text{shift}} \gets \text{LowRankProject}(h)$
\State \textbf{Modulation:}
\State $A_{\log}' \gets A_{\log} + \beta_A \tanh(A_{\text{shift}})$
\State $\Delta' \gets \Delta + \beta_\Delta \tanh(\Delta_{\text{shift}})$
\State $B' \gets B \odot (1 + \alpha_B \tanh(B_{\text{scale}})) + \beta_B \tanh(B_{\text{shift}})$
\State $C' \gets C \odot (1 + \alpha_C \tanh(C_{\text{scale}})) + \beta_C \tanh(C_{\text{shift}})$
\State \textbf{Return} modulated parameters
\end{algorithmic}
\end{algorithm}

\subsubsection{Chebyshev Polynomial Adapter}
The Chebyshev polynomial adapter imposes structural smoothness by mapping conditional modulations onto orthogonal polynomial bases. Rather than predicting unbounded shifts directly, the hidden dimension governs a latent evaluation coordinate $z \in [-1, 1]$ and generates series coefficients. This creates a low-degree global approximation to the operator path, which is appropriate when the desired recurrent dynamics vary smoothly with context. The same global support also means that changing one coefficient can affect the modulation over the full coordinate range, so the degree $K$ controls a bias-variance tradeoff rather than merely increasing capacity.

\begin{algorithm}[H]
\caption{Chebyshev \marksystem{} Kernel}
\label{alg:chebyshev_kernel}
\begin{algorithmic}[1]
\State \textbf{Input:} Context vector $\cvector$, Base parameters $A_{\log}, B, C, \Delta, D$
\State \textbf{Hyperparameters:} Degree $K$, Scaling constraints $\alpha, \beta$
\State $h \gets \text{MLP}(\cvector)$
\State $z \gets \text{Linear}_{z}(h)$ \Comment{Project hidden state to latent basis coordinate}
\State $T(z) \gets [T_0(z), T_1(z), \dots, T_{K-1}(z)]$ \Comment{Evaluate Chebyshev polynomials up to degree $K$}
\State $\theta_{\text{coeffs}} \gets \text{Linear}_{\theta}(h) \quad \text{for } \theta \in \{A, \Delta, D\}$
\State $\theta_{\text{shift}} \gets \sum_{k=0}^{K-1} \theta_{\text{coeffs}}^{(k)} \cdot T_k(z) \quad \text{for } \theta \in \{A, \Delta, D\}$
\State $\bar{T} \gets \frac{1}{n_{\text{heads}}} \sum_i T(z_i)$ \Comment{Calculate shared basis mean for broad matrix shifts}
\State $X_{\text{scale}} \gets \text{UV}_{X_{\text{scale\_base}}}(h) + \sum_{k=0}^{K-1} \text{UV}_{X_{\text{scale\_basis}}}(h)^{(k)} \cdot \bar{T}_k \quad \text{for } X \in \{B, C\}$
\State $X_{\text{shift}} \gets \text{UV}_{X_{\text{shift\_base}}}(h) + \sum_{k=0}^{K-1} \text{UV}_{X_{\text{shift\_basis}}}(h)^{(k)} \cdot \bar{T}_k \quad \text{for } X \in \{B, C\}$
\State \textbf{Modulation:} Evaluate modulated representations using base parameters and $X_{\text{scale/shift}}$ via bounded $\tanh$ saturation (Step 8-11 in Algorithm~\ref{alg:hypernet_kernel})
\State \textbf{Return} modulated parameters
\end{algorithmic}
\end{algorithm}

\subsubsection{Discrete Cosine Transform (DCT) Kernel}
The DCT kernel imposes a spectral inductive bias by casting modulations as bandlimited real-valued basis transformations. Rather than assuming that high-frequency modulation is always harmful, the parameterization makes high-frequency content explicit and attenuates it through a learned monotone decay. Drawing inspiration from the Inverse Real Fast Fourier Transform (IRFFT) without requiring unconstrained complex arithmetic, we formulate a deterministic real-valued basis grid. Contextual representations evaluate spectral components that decay across higher frequencies before being mapped back into parameter shifts. This makes DCT most natural when the useful operator variation is expected to be coarse-to-fine or spectrally organized; if the true modulation requires sharp localized changes, the frequency budget and decay prior can become limiting factors.

\begin{algorithm}[H]
\caption{DCT \marksystem{} Kernel}
\label{alg:dct_kernel}
\begin{algorithmic}[1]
\State \textbf{Input:} Context vector $\cvector$, Base parameters $A_{\log}, B, C, \Delta, D$
\State \textbf{Hyperparameters:} NumFreqs $n_f$, Basis grid timepoints $L$, Decay bound $\alpha_{\text{decay}}$
\State $h \gets \text{MLP}(\cvector)$
\State $w_{\text{decay}}(f) \gets \frac{1}{(f+1)^{\text{softplus}(\alpha) + \epsilon}}$ \Comment{Evaluate learned frequency attenuation}
\State $\text{Spec}_{\text{weights}} \gets \text{Linear}_{\text{spec}}(h) \odot w_{\text{decay}}$ 
\State $\text{Time}_{\text{vec}} \gets \text{Spec}_{\text{weights}} \times \mathcal{B}_{DCT-II}$ \Comment{Project bandlimited spectral amplitudes into temporal structure}
\State $\theta_{\text{shift}} \gets \text{Linear}_{\theta}(\text{Time}_{\text{vec}}) \quad \text{for } \theta \in \{A, \Delta, D\}$
\State $X_{\text{scale}} \gets \text{UV}_{X_{\text{scale\_base}}}(h) + \text{UV}_{X_{\text{scale\_basis}}}(\text{Time}_{\text{vec}}) \quad \text{for } X \in \{B, C\}$
\State $X_{\text{shift}} \gets \text{UV}_{X_{\text{shift\_base}}}(h) + \text{UV}_{X_{\text{shift\_basis}}}(\text{Time}_{\text{vec}}) \quad \text{for } X \in \{B, C\}$
\State \textbf{Modulation:} Evaluate modulated representations using base parameters and $X_{\text{scale/shift}}$ via bounded $\tanh$ saturation (Step 8-11 in Algorithm~\ref{alg:hypernet_kernel})
\State \textbf{Return} modulated parameters
\end{algorithmic}
\end{algorithm}

Overall, the kernel choice determines how much structure is imposed before the bounded AQS-compatible shifts are applied. Hypernet may overfit or produce rapidly varying modulations when the available data do not justify its flexibility. Chebyshev may oversmooth localized changes because its polynomial basis has global support. DCT may underfit non-bandlimited operator paths if too few frequencies are retained or if the learned decay is too strong. These are not separate stability assumptions, since all variants still use bounded shifts, but they are practical modeling assumptions that affect how efficiently \marksystem{} uses its trainable degrees of freedom.

\subsection{CART Objective}\label{app:cart_details}
To rigorously adapt the training criteria against distributional mode collapse in heavily masked regimes, we replace standard Categorical Cross-Entropy (CCE) masked language modeling loss with the Context-Adaptive Token-Level Noise Rescheduling (CART) objective. In conventional sequence modeling, heavy masking at early diffusion steps often degenerates the output tensor into identical redundant tokens predicting the global marginal distribution rather than contextualized completions. To prevent mode collapse, the objective introduces differentiated spatial penalty terms, adapted from Dream \cite{ye2025dream7bdiffusionlarge}. As a consequence, when evaluating with a global context such as MDLM ELBO \cite{schiff2025simple}, the perplexities  may seem inflated.

We provide the full mathematical specification of the Context-Adaptive Token-level Noise Rescheduling (CART) objective described in \S\ref{subsec:cart}, including the token-wise weighting function and its substitution into the masked diffusion loss.

Let $m_i=\mathbf{1}[x_{t,i}=\mathcal{M}]$ indicate which tokens are masked at timestep $t$. Following Dream \cite{ye2025dream7bdiffusionlarge}, CART replaces the single sequence-level weight with a per-token weight built from a symmetric geometric kernel over token distance. For a masked token $i$,
\begin{equation}
w_i(t,\mathbf{x}_t)
=
\frac{1}{2}\sum_{j=1}^{L} (1-m_j)\; \operatorname{Geo}\!\big(p,\,|i-j|-1\big),
\qquad
\operatorname{Geo}(p,d)=p\,(1-p)^{d},
\label{eq:cart_weight}
\end{equation}
where the sum runs over the visible (unmasked) tokens and $p \in (0,1]$ controls the kernel sharpness. A small $p$ spreads each clean token's influence almost uniformly across the sequence, while a large $p$ concentrates it on nearby masked tokens. Because the weight is the sum of one geometric kernel per visible token, a masked position surrounded by many nearby visible tokens receives a larger weight; this is the mixture of geometric kernels over clean anchors used in Dream. We use $p=0.45$ (Appendix~\ref{app:hyperparams}). Substituting $w_i(t,\mathbf{x}_t)$ into Eq.~\eqref{eq:masked_diff_loss} yields the CART-weighted training objective:
\begin{equation}
\mathcal{L}_{\text{CART}}(\theta)
=
\mathbb{E}_{\mathbf{x}_0,\,t,\,\mathbf{x}_t}
\Big[
\sum_{i=1}^{L} m_i\, w_i(t,\mathbf{x}_t)\; \big(-\log p_\theta(x_{0,i}\mid \mathbf{x}_t, t)\big)
\Big].
\label{eq:cart_loss}
\end{equation}

\subsection{Scope of Contribution} \label{app:contribution}

Before discussing limitations, we clarify the scope of what this paper establishes.
Our contribution is a theoretically grounded formulation of \textbf{Dynamic Operator Modulation} for SSMs: instead of injecting conditioning information only through inputs or activation scaling, \marksystem{} maps a context vector into bounded changes of the recurrent operator itself.
The comparison to Adaptive Layer Normalization (AdaLN) and input-stream injection in Section~\ref{sec:introduction} is a structural motivation, not a claim that \marksystem{} dominates those mechanisms in all empirical settings.

\textbf{Scope of the conditioning signal.}
The operator-conditioning claim is scoped to a continuous, scalar, ordered scheduling variable, which in this work is the diffusion timestep $t \in [0,1]$.
This setting matters because adjacent timesteps index nearly identical denoising tasks, so the context-to-operator map should vary smoothly with the coordinate. Chebyshev and DCT encode that smoothness structurally, which is the mechanism behind their advantage over the Hypernet adapter (Section~\ref{subsec:ablations}).
The \marksystem{} architecture accepts any embedding as $\cvector$, but the smooth-operator prior is specific to a continuous and ordered coordinate, so we do not claim transfer to autoregressive next-token modeling, which has no denoising coordinate to schedule against, or to high-dimensional prompt conditioning, where the signal is discrete and carries no proximity metric. Those regimes are left to future work.

The central technical challenge is that Dynamic Operator Modulation is a meta-function learning problem: a conditioning vector $\cvector$ must determine continuous-time system parameters $(A, B, C, D, \Delta)$ that control long-range temporal behavior.
This resembles the setting studied by hypernetworks~\citep{ha2017hypernetworks}, where a side network produces parameters for a primary model.
In recurrent or state-space systems, such context-to-operator maps raise three practical risks:
\begin{enumerate}[label=(\roman*)]
    \item \textbf{Training instability and divergence.}
    If generated recurrence parameters are unconstrained, small context-dependent errors can move eigenvalues toward unstable regions, producing unstable forward dynamics or poorly conditioned gradients~\citep{ha2017hypernetworks, haber2017stable}.
    \item \textbf{Identifiability collapse.}
    A low-dimensional conditioning signal can map to many internally different state-space parameterizations that represent the same input-output behavior, and poorly structured maps can become nearly insensitive to $\cvector$~\citep{galanti2020modularity}.
    \item \textbf{Optimization burden.}
    Emitting large dense parameter tensors from a small context vector can introduce high-variance gradients and make the context-to-operator map difficult to train reliably~\citep{ha2017hypernetworks}.
\end{enumerate}

\marksystem{} addresses these risks through a constrained architecture rather than by generating unrestricted recurrent weights.
\begin{itemize}
    \item \textbf{Stability} is controlled by bounded log-space modulations and positive discretization, yielding the constructive AQS certificate in Section~\ref{subsec:aqs} and Appendix~\ref{app:proofs} under the stated parameterization~\citep{cox2018lpv}.
    \item \textbf{Operator recovery} is studied at the level of the coordinate-invariant Markov parameter sequence. The synthetic LPV experiments support recovery of this input-output operator under matched assumptions, while the theory states the corresponding assumptions explicitly.
    \item \textbf{Trainability} is encouraged by low-dimensional outputs: the adapters predict scalar shifts and low-rank scales rather than full dense matrices, using 6.3--11M trainable parameters on top of a frozen 111M-parameter backbone.
\end{itemize}
Thus, the contribution of this work is the formulation and initial validation of an operator-level conditioning paradigm for SSMs.
The parameter ablation in Table~\ref{tab:loo}, the inference-overhead comparison in Table~\ref{tab:latency}, and the Markov-operator diagnostics in Section~\ref{subsec:dynamics} provide mechanistic and systems-level evidence that the conditioning signal reaches the recurrent operator rather than only external activations. None of these is a matched-budget quality comparison against external conditioning methods, and Appendix~\ref{app:limitations} states what remains open.

\subsection{Limitations and Future Directions} \label{app:limitations}
This work establishes \marksystem{} as a stable and identifiable operator-level conditioning framework for SSMs, but its empirical scope is intentionally narrower than the full space of conditioning methods. Our experiments validate internal \marksystem{} design choices (kernel geometry in Table~\ref{tab:ablation}, SSM-parameter coverage in Table~\ref{tab:loo}), compare systems cost against external conditioning methods (Table~\ref{tab:latency}), and diagnose changes in the Markov parameter sequence (Figure~\ref{fig:markov_norm_lag}). The comparison against AdaLN and input-stream injection is limited to inference latency, memory, and parameter count. We do not provide a controlled head-to-head baseline against those methods under matched training compute, objective, and frozen-backbone constraints, so we make no claim that operator modulation outperforms them on reconstruction quality.

The present results also do not establish scaling laws for Dynamic Operator Modulation. All reported language-model experiments use a 111M-parameter Hydra-BERT backbone with 6.3--11M trainable adapter parameters. This setting is sufficient to test whether bounded operator modulation is feasible, stable, and trainable, but it does not determine how \marksystem{} behaves as model depth, hidden width, adapter rank, sequence length, training tokens, or pretrained SSM family vary. In particular, the Markov-dynamics diagnostics in Figure~\ref{fig:markov_norm_lag} suggest that \marksystem{} can reshape long-lag memory profiles, but they should be interpreted as mechanistic evidence rather than as a completed long-context scaling study.

Our evaluation is likewise centered on the claims of this paper: AQS certificates, synthetic LPV recovery, Markov-operator diagnostics, and CART-weighted validation loss. These measurements support stability, recoverability under matched assumptions, and timestep-conditioned operator adaptation. We now report inference systems metrics (Table~\ref{tab:latency}), but the evaluation does not cover generation quality, downstream task accuracy, or robustness across long-context language benchmarks. Future work should test \marksystem{} on controlled long-range suites such as LRA and language long-context evaluations, and should report both quality and systems metrics there.

The theoretical guarantees depend on assumptions that may be violated in less structured settings. The identifiability and generalization arguments assume a well-specified bounded LPV-SSM class, stable Markov operators, sufficiently exciting inputs, sub-Gaussian noise, and mixing conditions. The written stability proof relies on the bounded MaRK parameterization and constructive ZOH chain, while the statistical-learning proof composes standard concentration, mixing, and Rademacher-complexity results cited from the literature. These assumptions are appropriate for the synthetic recovery protocol and for the stability claims made here, but they should not be read as unconditional guarantees for arbitrary pretrained sequence models or arbitrary conditioning distributions.

Finally, the bounded parameterization that makes \marksystem{} stable may also constrain expressivity. Performance can plausibly depend on the modulation bounds, the low-rank dimension, the Chebyshev degree, the DCT frequency count, and the diffusion weighting. A useful next stage is a systematic robustness study over these hyperparameters. The conditioning regime itself is also open: within this paper the scheduling variable is the diffusion timestep, and whether operator-level conditioning helps when the conditioning signal is not a continuous and ordered coordinate remains future work (Appendix~\ref{app:contribution}).

\section{Proofs}
\label{app:proofs}

This section gives the written theoretical arguments for the strictly bounded \marksystem{} kernels. We first state the AQS stability argument for bounded affine parameter variations, then prove the coordinate-invariant operator identifiability and statistical convergence guarantee for Proposition~\ref{prop:identifiability}.

\subsection{Proof of Affine Quadratic Stability (AQS)}
Let the base matrix $A \in \mathbb{R}^{n \times n}$ and perturbed system $A(v) = A + \Delta A(v)$, where parameter variations $v$ lie in a hyper-rectangle $\Psi$. 

\textbf{Theorem.} The LPV system is Affine Quadratically Stable if there exists a matrix $P \succ 0$ and scalar $\varepsilon > 0$ such that for all vertices $v_k \in \Psi$, $A(v_k)^\top P A(v_k) - P \preceq -\varepsilon I$.

\begin{proof}
Because the mapping $\cvector \mapsto A(v)$ is modulated via strictly bounded differentiable 
activations ($\tanh$), the range of possible evaluations falls strictly within the convex hull 
defined by vertices $v_k$. This vertex argument applies to affine dependence on the bounded modulation variables; in the diagonal Hydra ZOH instantiation, the certificate is obtained directly with $P=I$ because $\bar{A}_i=\exp(-\exp(A'_{\log,i})\Delta_i)\in(0,1)$. For any affine interpolation $A(\lambda)=\sum_k\lambda_k A_k$, matrix convexity of $A\mapsto A^\top P A$ for $P\succ0$ gives
\[
A(\lambda)^\top P A(\lambda)\preceq \sum_k \lambda_k A_k^\top P A_k .
\]
Thus, if the strictly negative definite LMI $A_k^\top P A_k - P \preceq -\varepsilon I$ is satisfied uniformly at all finite vertices, every interpolated system inherits the same Lyapunov decrease. Hence $V(x) = x^\top P x$ decays exponentially.
\end{proof}

\subsection{Proof of Operator Identifiability (Proposition 4.1)}
\label{app:operator_identifiability_proof}

We prove the result for the coordinate-invariant input-output operator.  The
direct term \(D_\theta(c)\) is not part of the state-coordinate ambiguity, but
it is part of the one-step input-output map; therefore the statistical bound is
stated for the extended Markov operator including \(D\).  If \(D\) is fixed by
the model class, the \(D\)-term below simply drops out.

Let \(\mathcal{C}\subset\mathbb{R}^{d_c}\) be compact, let
\(\Theta\subset\mathbb{R}^d\) be the compact trainable parameter set, and write
\[
G_\theta(c)=\big(A_\theta(c),B_\theta(c),C_\theta(c),D_\theta(c)\big).
\]
All matrix norms are operator norms unless \(\|\cdot\|_F\) is written.  Define
\[
\begin{aligned}
M_B&:=\sup_{\theta,c}\|B_\theta(c)\|,\qquad
M_C:=\sup_{\theta,c}\|C_\theta(c)\|,\qquad
M_D:=\sup_{\theta,c}\|D_\theta(c)\|,\\
D_\Theta&:=\sup_{\theta,\theta'\in\Theta}\|\theta-\theta'\|_2 .
\end{aligned}
\]
The \(\tanh\)-bounded adapter is Lipschitz on
\(\Theta\times\mathcal{C}\).  Let
\[
L_A:=\sup_{\theta\ne\theta',c}
\frac{\|A_\theta(c)-A_{\theta'}(c)\|}{\|\theta-\theta'\|_2},
\qquad
L_B,\ L_C,\ L_D
\quad\text{be defined analogously}.
\]
These constants are finite because \(\Theta\) and \(\mathcal{C}\) are compact,
the adapter is a finite composition of affine maps and \(\tanh\), and
\(\|\tanh'\|_\infty\le1\).  Let \(n_x,n_u,n_y\) denote the state, input, and
output dimensions.

\paragraph{AQS constants.}
Affine quadratic stability gives \(P=P^\top\succ0\) and \(\varepsilon>0\) such
that, for every \(\theta\in\Theta\) and \(c\in\mathcal{C}\),
\begin{equation}
A_\theta(c)^\top P A_\theta(c)-P\preceq-\varepsilon I .
\label{eq:aqs_uniform_lyapunov}
\end{equation}
Let
\[
\lambda_-:=\lambda_{\min}(P),\qquad
\lambda_+:=\lambda_{\max}(P),\qquad
\rho:=\sqrt{1-\varepsilon/\lambda_+}\in(0,1),\qquad
M_A:=\sqrt{\lambda_+/\lambda_-}.
\]
The exogenous inputs \(u_t\in\mathbb{R}^{n_u}\) and output noise
\(\xi_t\in\mathbb{R}^{n_y}\) are centered and sub-Gaussian:
\[
\|u_t\|_{\psi_2}\le K_u,\qquad
\|\xi_t\|_{\psi_2}\le K_\xi .
\]
For the bounded squared-loss theorem used below, assume the synthetic
identification protocol has almost-sure radii
\(\|u_t\|_2\le R_u\) and \(\|\xi_t\|_2\le R_\xi\).  For unbounded
sub-Gaussian inputs, the same proof applies after the standard truncation
argument with \(R_u,R_\xi\) replaced by the corresponding high-probability
sub-Gaussian radii.

For a frozen context \(c\), define
\[
H_{\theta,k}(c):=C_\theta(c)A_\theta^k(c)B_\theta(c),\qquad k\ge0,
\]
and the extended Markov operator
\[
\widetilde{\mathcal{H}}_\theta(c)
:=\big(D_\theta(c),H_{\theta,0}(c),H_{\theta,1}(c),\ldots\big).
\]
The well-specified observation model is
\begin{equation}
y_t
=
D_{\theta_\star}(c_t)u_t
+\sum_{k=0}^{\infty}H_{\theta_\star,k}(c_t)u_{t-k-1}
+\xi_t,
\qquad \theta_\star\in\Theta .
\label{eq:lpv_observation_model}
\end{equation}
The context is frozen for the local Markov operator in
\eqref{eq:lpv_observation_model}.  If contexts vary inside the recurrence, the
identified object is instead the time-varying product
\(C(c_t)A(c_t)\cdots A(c_{t-k+1})B(c_{t-k})\).

Persistent excitation is required at the same operator level: there exists
\(\kappa_{\rm pe}>0\) such that, for every square-summable measurable sequence
\(\{G_k(c)\}_{k\ge -1}\),
\begin{equation}
\mathbb{E}\left\|
G_{-1}(c_t)u_t+\sum_{k=0}^{\infty}G_k(c_t)u_{t-k-1}
\right\|_2^2
\ge
\kappa_{\rm pe}\,
\mathbb{E}_{c}\left[
\|G_{-1}(c)\|_F^2+\sum_{k=0}^{\infty}\|G_k(c)\|_F^2
\right].
\label{eq:persistent_excitation}
\end{equation}

\begin{lemma}[Coordinate-invariant equivalence class]
\label{lem:markov_equivalence_class}
Let \(\mathcal{V}\) be the context index set.  For two LPV realizations
\((A,B,C,D)\) and \((A',B',C',D')\), define
\[
(A,B,C,D)\sim(A',B',C',D')
\]
if and only if \(D(v)=D'(v)\) and
\[
C(v)A^k(v)B(v)=C'(v)A'^{\,k}(v)B'(v),
\qquad \forall v\in\mathcal{V},\ k\in\mathbb{N}_0 .
\]
Then \(\sim\) is an equivalence relation, and each equivalence class is
uniquely represented by the extended Markov operator
\(\widetilde{\mathcal{H}}(v)=\big(D(v),\{C(v)A^k(v)B(v)\}_{k\ge0}\big)\).
\end{lemma}

\begin{proof}
Reflexivity and symmetry are immediate.  Transitivity follows because equality
of matrices is transitive for every \(v\) and \(k\).  Hence the quotient
identifies exactly those state-coordinate realizations that induce the same
input-output Markov parameters.  Conversely, equality of input-output maps for
all finitely supported inputs identifies \(D(v)\) from the instantaneous
response and identifies \(C(v)A^k(v)B(v)\) by applying an impulse at lag
\(k+1\).  Under the stochastic experiment, the same implication holds in
\(L_2\) by persistent excitation~\eqref{eq:persistent_excitation}.  Thus the
identifiable object is the structural equivalence class represented by
\(\widetilde{\mathcal{H}}(v)\), not a particular choice of state coordinates.
\end{proof}

\begin{lemma}[AQS implies geometric Markov decay]
\label{lem:aqs_markov_decay}
For every \(\theta\in\Theta\), \(c\in\mathcal{C}\), and \(k\ge0\),
\begin{equation}
\|A_\theta^k(c)\|\le M_A\rho^k,
\qquad
\|H_{\theta,k}(c)\|
\le
M_C M_A M_B \rho^k .
\label{eq:markov_decay_bound}
\end{equation}
The same bound holds for time-varying products
\(\Phi_{\theta,t,k}=A_\theta(c_t)A_\theta(c_{t-1})\cdots
A_\theta(c_{t-k+1})\).
\end{lemma}

\begin{proof}
Let \(\|x\|_P^2=x^\top Px\).  From
\eqref{eq:aqs_uniform_lyapunov},
\[
\|A_\theta(c)x\|_P^2
\le
\|x\|_P^2-\varepsilon\|x\|_2^2
\le
\left(1-\varepsilon/\lambda_+\right)\|x\|_P^2
=\rho^2\|x\|_P^2 .
\]
Iterating gives \(\|A_\theta^{\,k}(c)x\|_P\le\rho^k\|x\|_P\), and converting
between the Euclidean and \(P\)-norms yields
\(\|A_\theta^k(c)\|\le\sqrt{\lambda_+/\lambda_-}\rho^k=M_A\rho^k\).
The common Lyapunov matrix \(P\) gives the identical estimate for arbitrary
time-varying products.  Multiplying by the uniform \(B\)- and \(C\)-bounds gives
\eqref{eq:markov_decay_bound}.
\end{proof}

\begin{lemma}[Geometric decay implies exponential \(\beta\)-mixing]
\label{lem:geometric_decay_mixing}
Let \(Z_t=(c_t,u_t,\xi_t,y_t)\) be the stationary process generated by
\eqref{eq:lpv_observation_model}.  Suppose the exogenous process
\((c_t,u_t,\xi_t)\) is a causal process of innovations with density bounded by
\(K_p<\infty\) and absolute-regularity coefficients
\(\beta_{\rm exo}(q)\le b_0e^{-a_0q}\); the independent-exogenous case is
obtained by taking \(b_0=0\).  Then \(Z_t\) is exponentially
\(\beta\)-mixing: there exist constants \(b_\beta<\infty\) and
\(\rho_\beta\in(0,1)\), depending only on
\((M_A,M_B,M_C,M_D,\rho,R_u,R_\xi,b_0,a_0,K_p,n_u,n_y,d_c)\), such that
\begin{equation}
\beta_Z(q)\le b_\beta \rho_\beta^q,\qquad q\ge1 .
\label{eq:beta_mixing_rate}
\end{equation}
One admissible choice is
\[
\rho_\beta:=\max\{\rho^{1/2},e^{-a_0/2}\},
\quad
b_\beta:=b_0+
C_{\rm PT}(K_p,n_u,n_y,d_c)
\left(1+\frac{M_C M_A M_B R_u}{1-\rho}+M_D R_u+R_\xi\right)
\]
where \(C_{\rm PT}\) is the finite constant in the Pham--Tran
absolute-regularity theorem for Volterra/Bernoulli-shift processes
\cite{pham1985mixing}.
\end{lemma}

\begin{proof}
Let \(y_t^{(q)}\) be the output obtained by truncating
\eqref{eq:lpv_observation_model} to inputs no older than \(q\) lags.  By
Lemma~\ref{lem:aqs_markov_decay},
\begin{equation}
\|y_t-y_t^{(q)}\|_2
\le
\sum_{k>q}\|H_{\theta_\star,k}(c_t)\|\,\|u_{t-k-1}\|_2
\le
\frac{M_C M_A M_B R_u}{1-\rho}\rho^{q+1}.
\label{eq:truncation_tail}
\end{equation}
Thus \(Z_t\) is a causal Bernoulli-shift, equivalently a Volterra process, whose
remote-past influence is exponentially summable.  The theorem of Pham and Tran
\cite{pham1985mixing} states that a Volterra process with bounded-density
innovations and summable kernel tails is absolutely regular, with mixing
coefficients controlled by the innovation mixing coefficient plus the kernel
tail.  Combining \(\beta_{\rm exo}(q)\le b_0e^{-a_0q}\) with
\eqref{eq:truncation_tail} and enlarging constants gives
\eqref{eq:beta_mixing_rate}.
\end{proof}

\begin{lemma}[Dependent Rademacher complexity]
\label{lem:dependent_rademacher}
For \(z_t=(c_t,\{u_{t-j}\}_{j\ge0},y_t)\), define
\[
f_\theta(z_t)
=
D_\theta(c_t)u_t+\sum_{k=0}^{\infty}H_{\theta,k}(c_t)u_{t-k-1},
\qquad
\ell_\theta(z_t):=\|y_t-f_\theta(z_t)\|_2^2 .
\]
Let \(\mathcal{L}:=\{\ell_\theta:\theta\in\Theta\}\).  Define
\[
B_f:=\left(M_D+\frac{M_C M_A M_B}{1-\rho}\right)R_u,
\qquad
B_y:=B_f+R_\xi,
\qquad
B_\ell:=4B_y^2,
\]
and
\[
L_f:=R_u\left[
L_D
+\frac{L_C M_A M_B+M_C M_A L_B}{1-\rho}
+\frac{M_C M_B L_A M_A^2}{(1-\rho)^2}
\right].
\]
Finally let
\[
\begin{aligned}
\Gamma_\beta
&:=1+2\sum_{q=1}^{\infty}\beta_Z(q)
\le
1+\frac{2b_\beta\rho_\beta}{1-\rho_\beta},\\
J_d
&:=\int_0^1
\sqrt{d\log\left(1+\frac{3L_fD_\Theta}{B_f s}\right)}\,ds .
\end{aligned}
\]
Then the empirical Rademacher complexity of the squared-loss class satisfies
\begin{equation}
\mathfrak{R}_N^\beta(\mathcal{L})
\le
\frac{C_{\rm rad}}{\sqrt{N}},
\label{eq:rademacher_rate}
\end{equation}
where
\[
C_{\rm rad}
:=
96\,B_yB_f\sqrt{n_y\Gamma_\beta}\,J_d .
\]
\end{lemma}

\begin{proof}
The geometric series in Lemma~\ref{lem:aqs_markov_decay} gives
\(\|f_\theta(z_t)\|_2\le B_f\) and \(\|y_t\|_2\le B_y\), hence
\(0\le\ell_\theta\le B_\ell\).  For any \(\theta,\theta'\),
\[
\begin{aligned}
\|H_{\theta,k}(c)-H_{\theta',k}(c)\|
&\le
\Big(L_C M_A M_B\rho^k
      +M_C M_A L_B\rho^k  \\
&\qquad
      +M_C M_B L_A M_A^2 k\rho^{k-1}\Big)
      \|\theta-\theta'\|_2 ,
\end{aligned}
\]
and summing over \(k\ge0\), adding the \(D\)-term, and multiplying by \(R_u\)
gives
\[
\|f_\theta(z_t)-f_{\theta'}(z_t)\|_2\le
L_f\|\theta-\theta'\|_2 .
\]
Thus the covering number of each output coordinate satisfies
\[
\log \mathcal{N}(\epsilon,\{f_{\theta,j}\},L_2)
\le
d\log\left(1+\frac{3L_fD_\Theta}{\epsilon}\right).
\]
Dudley's entropy integral gives
\[
\mathfrak{R}_N(\{f_{\theta,j}\})
\le
\frac{12B_f}{\sqrt{N}}J_d .
\]
Summing over \(n_y\) output coordinates and using Cauchy--Schwarz gives the
\(\sqrt{n_y}\) factor.  Since
\(\ell_\theta(z)=\|y-f_\theta(z)\|_2^2\) is \(4B_y\)-Lipschitz in
\(f_\theta(z)\) on the bounded range, Talagrand's contraction principle
\cite{LedouxTalagrand1991,bartlett2002rademacher} transfers the predictor
bound to the squared-loss class.  For the dependent sequence, the
independent-block coupling of Yu and the Mohri--Rostamizadeh blocking lemma for
\(\beta\)-mixing sequences inflate the independent complexity by at most
\(\sqrt{\Gamma_\beta}\) \cite{yu1994rates,mohri2008rademacher}.  This yields
\eqref{eq:rademacher_rate}.
\end{proof}

\begin{proposition}[Excess-risk and operator convergence]
\label{prop:appendix_identifiability_convergence}
Let \(\hat{\theta}\) be an empirical risk minimizer over \(\Theta\):
\[
\hat{\theta}\in\arg\min_{\theta\in\Theta}
\widehat{\mathcal{R}}_N(\theta),
\qquad
\widehat{\mathcal{R}}_N(\theta):=\frac1N\sum_{t=1}^N\ell_\theta(z_t),
\]
and let \(\mathcal{R}(\theta):=\mathbb{E}\ell_\theta(Z_t)\).  For any
\(\delta\in(0,1)\), with probability at least \(1-\delta\),
\begin{equation}
\mathcal{R}(\hat{\theta})-\mathcal{R}(\theta_\star)
\le
\frac{C_{\rm ex}(\delta)}{\sqrt{N}},
\label{eq:excess_risk_rate}
\end{equation}
where
\[
C_{\rm ex}(\delta)
:=
4C_{\rm rad}
+2B_\ell\sqrt{2\Gamma_\beta\log(2/\delta)} .
\]
Moreover, the squared extended Markov-operator error obeys
\begin{equation}
\mathbb{E}_{c}\left[
\|D_{\hat{\theta}}(c)-D_{\theta_\star}(c)\|_F^2
+\sum_{k=0}^{\infty}
\|H_{\hat{\theta},k}(c)-H_{\theta_\star,k}(c)\|_F^2
\right]
\le
\frac{C_{\rm ex}(\delta)}{\kappa_{\rm pe}\sqrt{N}} .
\label{eq:markov_operator_rate}
\end{equation}
\end{proposition}

\begin{proof}
For bounded losses generated by an exponentially \(\beta\)-mixing process,
the uniform convergence theorem for absolutely regular sequences gives
\cite{mohri2008rademacher,kuznetsov2017generalization}
\[
\sup_{\theta\in\Theta}
\left|\mathcal{R}(\theta)-\widehat{\mathcal{R}}_N(\theta)\right|
\le
2\mathfrak{R}_N^\beta(\mathcal{L})
+B_\ell\sqrt{\frac{2\Gamma_\beta\log(2/\delta)}{N}}
\]
with probability at least \(1-\delta\).  Substituting
Lemma~\ref{lem:dependent_rademacher} gives the right-hand side
\[
\Delta_N(\delta)
\le
\frac{2C_{\rm rad}
      +B_\ell\sqrt{2\Gamma_\beta\log(2/\delta)}}{\sqrt{N}} .
\]
Since \(\hat{\theta}\) minimizes the empirical risk and
\(\theta_\star\in\Theta\),
\[
\mathcal{R}(\hat{\theta})-\mathcal{R}(\theta_\star)
\le
2\Delta_N(\delta)
\le
\frac{C_{\rm ex}(\delta)}{\sqrt{N}},
\]
which proves~\eqref{eq:excess_risk_rate}.

It remains to connect the statistical risk to the coordinate-invariant
operator.  Because the noise is centered and independent of the inputs, and the
model is well specified, the noise variance cancels in the excess risk:
\[
\mathcal{R}(\theta)-\mathcal{R}(\theta_\star)
=
\mathbb{E}\left\|
\Delta D_\theta(c_t)u_t
+
\sum_{k=0}^{\infty}
\Delta H_{\theta,k}(c_t)u_{t-k-1}
\right\|_2^2 .
\]
Here \(\Delta D_\theta=D_\theta-D_{\theta_\star}\) and
\(\Delta H_{\theta,k}=H_{\theta,k}-H_{\theta_\star,k}\).  Applying persistent
excitation~\eqref{eq:persistent_excitation} with
\(G_{-1}=\Delta D_\theta\) and \(G_k=\Delta H_{\theta,k}\) yields
\[
\mathcal{R}(\theta)-\mathcal{R}(\theta_\star)
\ge
\kappa_{\rm pe}\,
\mathbb{E}_{c}\left[
\|\Delta D_\theta(c)\|_F^2
+\sum_{k=0}^{\infty}\|\Delta H_{\theta,k}(c)\|_F^2
\right].
\]
Taking \(\theta=\hat{\theta}\) and using
\eqref{eq:excess_risk_rate} proves~\eqref{eq:markov_operator_rate}.
\end{proof}

Combining Lemma~\ref{lem:markov_equivalence_class} with
Proposition~\ref{prop:appendix_identifiability_convergence} proves
Proposition~\ref{prop:identifiability}.  The internal matrices are identifiable
only modulo the structural equivalence relation \(\sim\), while the observable
Markov operator \(\mathcal{H}\) is identifiable.  Under squared loss, both the
excess prediction risk and the squared extended Markov-operator error are
bounded by \(C/\sqrt{N}\) with high probability, where \(C\) is an explicit
function of the AQS certificate \((P,\varepsilon)\), representation bounds
\((M_A,M_B,M_C,M_D)\), excitation and tail constants
\((K_u,K_\xi,R_u,R_\xi,\kappa_{\rm pe})\), the parameter dimension \(d\), the
\(\tanh\)-parameterization Lipschitz constants
\((L_A,L_B,L_C,L_D)\), the mixing constants \((b_\beta,\rho_\beta)\), and the
confidence level \(\delta\), but not of \(N\).  The unsquared Frobenius norm is
controlled by taking a square root of~\eqref{eq:markov_operator_rate}; the rate
stated in Proposition~\ref{prop:identifiability} is the squared-loss
estimation rate.

\section{Experiment Details}
\label{app:experiment_details}

\subsection{Training Details} \label{app:training_details}
Training adapts the frozen Hydra-BERT backbone to the masked diffusion objective on an internal Hydra corpus constructed from the source mixture in Table~\ref{tab:hydra_data_mixture}. The reported train and validation portions contain 4{,}295{,}110{,}656 tokens in total: 3{,}537{,}149{,}952 training tokens and 757{,}960{,}704 validation tokens. The validation split is used for monitoring and reporting the held-out losses in the experiments, while all model updates use only the training split.

\begin{table}[H]
\centering
\small
\caption{Source mixture used to construct the Hydra train/validation corpus. Percentages are the relative shard-pool weights of the reported training runs and normalize at sampling.}
\label{tab:hydra_data_mixture}
\begin{tabular}{lr}
\toprule
Source & Ratio (\%) \\
\midrule
common-pile StackExchange & 10.00 \\
common-pile LibreTexts & 1.67 \\
common-pile YouTube & 22.50 \\
common-pile PubMed & 7.50 \\
common-pile CCCC & 10.00 \\
common-pile Project Gutenberg & 5.00 \\
common-pile arXiv papers & 7.50 \\
common-pile News & 1.67 \\
common-pile DOAB & 15.00 \\
common-pile Pressbooks & 1.67 \\
WikiText-103 raw & 2.50 \\
common-pile Data Provenance Initiative & 8.33 \\
common-pile Wikimedia & 8.33 \\
common-pile arXiv abstracts & 3.33 \\
\bottomrule
\end{tabular}
\end{table}

For all reported experiments, only the \marksystem{} adapter and timestep-embedding parameter groups are trainable. The Hydra backbone parameters and masked-LM classification head remain fixed, isolating the effect of dynamic operator modulation. The objective is the masked-token diffusion negative log-likelihood with CART weighting enabled using parameter $p=0.45$. Training uses Distributed Data Parallel (DDP) on two NVIDIA GH200 GPUs with bfloat16 precision, gradient checkpointing, high-precision matrix multiplication, and gradient clipping at norm $1.0$.

\begin{table}[H]
\centering
\small
\caption{Training settings shared by all three \marksystem{} kernels. Values are taken from the training configurations and the corresponding optimization implementation.}
\label{tab:common_training}
\begin{tabular}{p{0.36\linewidth}p{0.54\linewidth}}
\toprule
Setting & Value \\
\midrule
Trainable parameters & \marksystem{} adapters and timestep embeddings only \\
Backbone & Frozen 23-layer Hydra-BERT SSM \\
Sequence length & 4096 tokens \\
Epochs & 3 \\
Optimizer & Fused AdamW, $\beta=(0.9,0.999)$, weight decay $0.01$ \\
MaRK learning rate & $6.0\times 10^{-4}$ \\
Learning-rate schedule & Step-wise linear warmup, then cosine decay \\
Minimum LR ratio & 0.1 \\
Precision & bfloat16 \\
Training strategy & Distributed Data Parallel \\
Workers & 8 \\
Matrix multiply precision & high \\
Hydra hidden size / vocabulary & 768 / 30{,}522 \\
SSM state / convolution / head dim & 64 / 7 / 64 \\
Expansion / chunk size & 2 / 256 \\
\bottomrule
\end{tabular}
\end{table}

The variants are initialized from their corresponding kernel-specific checkpoints. Apart from the kernel geometry, all reported variants use the same batch size, warmup length, total training horizon, optimizer, and base learning rate. The $6.0\times10^{-4}$ learning rate is linearly warmed up for 720 steps and then cosine-decayed over the remaining steps of the 14{,}400-step schedule to a minimum value of $0.1\times$ the base rate. The kernel-specific entries in Table~\ref{tab:kernel_training} specify only the adapter geometry: Chebyshev uses a degree-5 polynomial basis, while DCT uses a 256-point basis grid with 8 retained frequencies.

\begin{table}[H]
\centering
\small
\caption{Kernel-specific training settings. Shared kernel dimensions are rank $r=2$, adapter MLP dimension 256, and timestep embedding dimension 128.}
\label{tab:kernel_training}
\begin{tabular}{llll}
\toprule
Kernel & Batch size & Warmup / total steps & Kernel-specific parameters \\
\midrule
Hypernet & 112 & 720 / 14{,}400 & -- \\
Chebyshev & 112 & 720 / 14{,}400 & Polynomial degree $=5$ \\
DCT & 112 & 720 / 14{,}400 & Timepoints $=256$, frequencies $=8$ \\
\bottomrule
\end{tabular}
\end{table}

\subsection{Modulation Bounds and CART Settings}\label{app:hyperparams}
The bounded shifts and scales of Section~\ref{subsec:modulation} use the magnitudes in Table~\ref{tab:modulation_hyper}. Diagonal parameters use $\beta_\Theta \tanh(\cdot)$, with the $A$-shift bound capped at $\beta_A \le 0.5$ and the remaining magnitudes learned; dense parameters use $\alpha_\Phi$ and $\beta_\Phi$ on the low-rank factors. Every magnitude is initialized to $0.2$. The kernel dimensions (low-rank size, adapter MLP, timestep embedding, Chebyshev degree, and DCT frequency budget) are listed in Table~\ref{tab:kernel_training}.

\begin{table}[H]
\centering
\small
\caption{Modulation bounds and CART settings, identical across the three kernels.}
\label{tab:modulation_hyper}
\begin{tabular}{ll}
\toprule
Parameter & Value \\
\midrule
$A$-shift bound $\beta_A$ & $\le 0.5$ \\
Scale and shift initial magnitude $\alpha,\beta$ & $0.2$ (learned) \\
CART $p$ / scale & 0.45 / 1.0 \\
\bottomrule
\end{tabular}
\end{table}

The $A$-shift bound is the only constraint that enters the stability certificate, so it is capped while the remaining magnitudes are learned freely; $B$, $C$, and $D$ never enter the Lyapunov condition because they are not raised to powers. Because $A<0$ and $\Delta>0$, every discrete eigenvalue stays in $(0,1)$ under any finite bound, so the modulated recurrence is Schur-stable for each individual parameter value and the Markov sequence always decays. The finite bound buys a uniform positive margin rather than stability itself: the worst case sits at the vertex $\delta=-\beta_A$, and as $\beta_A$ grows, that eigenvalue approaches $1$ from below while the margin $\varepsilon = 1-\max_i \bar{A}_i^2$ shrinks toward zero. An unbounded shift would therefore not destabilize any single realization, but it would remove the uniform guarantee that Appendix~\ref{app:proofs} relies on, which is why the bound is kept finite. The bounds themselves were set conservatively to keep the modulated operators near the frozen backbone, and the CART sharpness $p=0.45$ follows Dream.

\subsection{Evaluation Protocol and Statistics}\label{app:eval_protocol}
Every language-model number in this paper is a CART-weighted validation loss, read as \texttt{weighted\_nll} from the evaluation harness; perplexity is only the exponential of that quantity and is not the reported metric. Each cell is a mean over $10$ seeds for the accuracy tables (Tables~\ref{tab:ablation} and~\ref{tab:loo}) and over $15$ seeds for the inference table (Table~\ref{tab:latency}), with seed $i$ fixed at $42 + 100i$. Intervals are standard-error-of-the-mean bands over seeds (sample standard deviation with $\mathrm{ddof}=1$ divided by $\sqrt{n}$): Tables~\ref{tab:ablation} and~\ref{tab:loo} report the normal approximation (mean $\pm\,1.96\times$SEM over $10$ seeds), while the inference table (Table~\ref{tab:latency}) reports the Student-$t$ interval with $14$ degrees of freedom over $15$ seeds. The synthetic LPV curves in Figure~\ref{fig:moe_identifiability} use the same seed and interval convention.

Validation consumes the global random-number stream on every batch, sampling both the diffusion timestep and the masking pattern, and building the model also draws from it because Hydra initializes its discretization bias randomly. The harness therefore captures the RNG state immediately after the model is built and restores it before each dataset, so the CART suite reproduces the leave-one-out draws and the WikiText loss matches the leave-one-out full row to the bfloat16 run-to-run noise. Because the largest validation split (arXiv) contains tens of thousands of packed sequences at length 4096, its evaluation is capped at 500 validation batches, while smaller datasets are evaluated in full.

\subsection{AQS Certificate Evaluation}
We evaluate the AQS certificate by directly checking discrete-time stability margins implied by the ZOH discretization for the learned diagonal recurrence. Using the identity Lyapunov function $P=I$, this yields a closed-form certificate that avoids SDP solving. We load the pretrained diagonal log-parameters $A_{\log}^{(\ell)}$ from the 23-layer Hydra checkpoint and evaluate the worst-case (least stable) parameter vertex by monotonicity of the map $A_{\log}\mapsto \bar{A}=\exp(-\exp(A_{\log})\Delta)$ under the bounded shift $A_{\log}\pm\beta_A$, using $\Delta_{\min}=0.001$:

\begin{algorithm}[H]
\caption{Constructive AQS Certificate via $P = I$ Lyapunov Function}
\label{alg:aqs_certificate}
\begin{algorithmic}[1]
\State \textbf{Input:} Pretrained $A_{\log}^{(\ell)}$ for layers $\ell = 1, \ldots, 23$; bound $\beta_A = 0.5$; $\Delta_{\min} = 0.001$
\State Set $P \gets I$ \Comment{Identity matrix (valid Lyapunov function for diagonal $A$ with $\rho < 1$)}
\For{each layer $\ell$, each head $i$}
    \State $A_{\log,i}^{\text{worst}} \gets A_{\log,i}^{(\ell)} - \beta_A$ \Comment{Worst-case vertex (least decay)}
    \State $A_i \gets -\exp(A_{\log,i}^{\text{worst}})$ \Comment{Strictly negative}
    \State $\bar{a}_i \gets \exp(A_{i} \cdot \Delta_{\min})$ \Comment{Discrete eigenvalue $\in (0, 1)$}
\EndFor
\State $\rho_{\max}^{(\ell)} \gets \max_i |\bar{a}_i|$ \Comment{Spectral radius at worst-case vertex}
\State $\varepsilon_\ell \gets 1 - (\rho_{\max}^{(\ell)})^2$ \Comment{Lyapunov stability margin}
\State \textbf{Verify:} $\varepsilon_\ell > 0$ for all $\ell$ \Comment{$\bar{A}^\top I\, \bar{A} - I \prec 0$ (neg.\ definite)}
\State \textbf{Return} $P = I$, $\{\varepsilon_\ell\}_{\ell=1}^{23}$ \Comment{AQS Certificate with per-layer margins}
\end{algorithmic}
\end{algorithm}

The procedure is conservative in three ways: it fixes the Lyapunov function to $P=I$, evaluates the least stable vertex $\delta=-\beta_A$, and uses the smallest discretization $\Delta_{\min}=0.001$, which places the discrete eigenvalue closest to one. Because $\beta_A$ is shared across the bounded kernels, the certificate depends only on the pretrained $A_{\log}$ and not on the adapter geometry, so a single pass over the backbone covers all three variants. The reported margins are the distance of the worst-case eigenvalue from the unit circle.

\subsection{Synthetic LPV Identifiability Algorithms}
For identifiability tests, both ground-truth generator models and target estimators use diagonal LPV-SSMs with state dimension $d_{\text{state}} = 4$, factorization rank $r = 2$, sequence length $L = 20$, and datasets $N \in \{50,100,200,400,800,1600\}$ over 15 random seeds. Inputs are sampled as $u \sim \mathcal{N}(0,1)$ and diffusion coordinates as $\tau \sim \mathcal{U}(0,1)$; scalar outputs are generated by the ground-truth system and corrupted with Gaussian noise $\sigma=0.1$. Chebyshev and DCT use a fixed readout $C=\mathbf{1}$ to remove final rotation/scaling ambiguity, while the Hypernet synthetic kernel also modulates $C(\tau)$ with bounded low-rank scale/shift terms.
This choice is without loss of generality for scalar readouts: any nonzero scalar $C$ can be absorbed by a linear reparameterization of the latent state, so fixing $C=\mathbf{1}$ isolates operator recovery rather than an arbitrary output scaling.

The continuous matrices returned by Algorithms~\ref{alg:hypernet_ssm}--\ref{alg:dct_ssm} are rolled out with the same scalar ZOH discretization used in the scripts:
\[
\bar{A}(\tau)=\exp(A(\tau)\Delta t),\qquad
\bar{B}(\tau)=\frac{\bar{A}(\tau)-1}{A(\tau)+10^{-7}}\odot B(\tau),
\qquad \Delta t=0.1 .
\]
The synthetic sequence is then generated by
\[
x_s=\bar{A}(\tau_s)\odot x_{s-1}+\bar{B}(\tau_s)u_s,\qquad
y_s=C(\tau_s)^\top x_s .
\]
Estimators are initialized from noisy copies of the ground-truth parameters and trained with Mean-Squared Error (MSE). Recovery is measured by frozen-time Markov operator error on a 200-point $\tau$ grid over lags $k=0,\ldots,14$, using $M_k(\tau)=C(\tau)^\top(\bar{A^k}(\tau)\odot \bar{B}(\tau))$; for Chebyshev and DCT, $C(\tau)=\mathbf{1}$ is constant.

The Hypernet generator is the most flexible synthetic parameterization: a sinusoidal embedding of $\tau$ is mapped by a four-layer SiLU MLP to bounded shifts of the recurrence, input map, and readout. The negative exponential form enforces $A(\tau)<0$ componentwise, and sigmoid/tanh bounds keep the modulation amplitude controlled. Unlike the basis-kernel variants below, Hypernet returns a time-varying $C(\tau)$ in addition to $A(\tau)$ and $B(\tau)$.

\begin{algorithm}[H]
\caption{Hypernet Synthetic LPV-SSM}
\label{alg:hypernet_ssm}
\begin{algorithmic}[1]
\State \textbf{Input:} Diffusion coordinate $\tau \in [0,1]$
\State \textbf{Parameters:} $A_{\log\_\mathrm{base}}$, low-rank base $U_{\mathrm{base}},V_{\mathrm{base}}$, scalars $\beta_A^{\mathrm{param}},\alpha_B,\beta_B,\alpha_C,\beta_C$, four-layer SiLU MLP, linear heads
\State $\mathrm{cond}(\tau) \gets [\sin(2\pi m\tau),\cos(2\pi m\tau)]_{m=0}^{63}$ \Comment{128-dimensional conditioning signal}
\State $h \gets \mathrm{MLP}_{4\times\mathrm{SiLU}}(\mathrm{cond}(\tau))$
\State $\mathrm{uvflat}(q) \gets \mathrm{vec}(U_qV_q)_{1:d_{\mathrm{state}}}$, where $q$ is split into rank-$r$ factors $(U_q,V_q)$ \Comment{matches the low-rank head flattening}
\State $\beta_A^{\mathrm{eff}} \gets 0.5\,\sigma(\beta_A^{\mathrm{param}})$, \quad $a_{\mathrm{shift}} \gets W_Ah$ \Comment{bounds the recurrence modulation}
\State $A(\tau) \gets -\exp(A_{\log\_\mathrm{base}}+\beta_A^{\mathrm{eff}}\tanh(a_{\mathrm{shift}}))$
\State $b_{\mathrm{base}} \gets \mathrm{vec}(U_{\mathrm{base}}V_{\mathrm{base}})_{1:d_{\mathrm{state}}}$
\State $b_{\mathrm{scale}}(\tau) \gets 1+\alpha_B\tanh(\mathrm{uvflat}(W_{B,\mathrm{scale}}h))$, \quad $b_{\mathrm{shift}}(\tau) \gets \beta_B\tanh(\mathrm{uvflat}(W_{B,\mathrm{shift}}h))$ \Comment{bounded scale/shift for $B$}
\State $B(\tau) \gets b_{\mathrm{base}}\odot b_{\mathrm{scale}}(\tau)+b_{\mathrm{shift}}(\tau)$
\State $c_{\mathrm{scale}}(\tau) \gets 1+\alpha_C\tanh(\mathrm{uvflat}(W_{C,\mathrm{scale}}h))$, \quad $c_{\mathrm{shift}}(\tau) \gets \beta_C\tanh(\mathrm{uvflat}(W_{C,\mathrm{shift}}h))$
\State $C(\tau) \gets \mathbf{1}\odot c_{\mathrm{scale}}(\tau)+c_{\mathrm{shift}}(\tau)$ \Comment{learned readout for the Hypernet variant}
\State \textbf{Return} $A(\tau), B(\tau), C(\tau)$
\end{algorithmic}
\end{algorithm}

Hypernet is the only synthetic generator with a learned, time-varying readout, so its recovery error also reflects the readout path and not just the recurrence and input maps. It is therefore the hardest of the three synthetic targets, and the basis kernels are compared against it rather than against an easier diagonal reference.

The Chebyshev generator uses a fixed polynomial basis after mapping $\tau$ to $z\in[-1,1]$. All time variation enters through bounded Chebyshev shifts of the log-recurrence and of the low-rank factors defining $B(\tau)$. The readout is fixed to $C=\mathbf{1}$, so the reported recovery error measures the induced Markov operator rather than an arbitrary rescaling of the output map.

\begin{algorithm}[H]
\caption{Chebyshev Synthetic LPV-SSM}
\label{alg:chebyshev_ssm}
\begin{algorithmic}[1]
\State \textbf{Input:} Diffusion coordinate $\tau \in [0,1]$
\State \textbf{Parameters:} $A_{\log\_\mathrm{base}}$, $U_{\mathrm{base}}$, $V_{\mathrm{base}}$, coefficients $A_{\mathrm{coeffs}},U_{\mathrm{coeffs}},V_{\mathrm{coeffs}}$, bounds $\beta_A,\beta_B$, basis count $K_{\mathrm{basis}}=3$
\State $z \gets 2\tau - 1$ \Comment{Chebyshev domain}
\State $T_0(z) \gets 1$, $T_1(z) \gets z$
\For{$k = 2$ \textbf{to} $K_{\mathrm{basis}}-1$}
    \State $T_k(z) \gets 2z \cdot T_{k-1}(z) - T_{k-2}(z)$
\EndFor
\State $A_{\mathrm{shift}} \gets \sum_{k=0}^{K_{\mathrm{basis}}-1} T_k(z) A_{\mathrm{coeffs}}^{(k)}$ \Comment{basis modulation of the log recurrence}
\State $U_{\mathrm{shift}} \gets \sum_{k=0}^{K_{\mathrm{basis}}-1} T_k(z) U_{\mathrm{coeffs}}^{(k)}$, \quad $V_{\mathrm{shift}} \gets \sum_{k=0}^{K_{\mathrm{basis}}-1} T_k(z) V_{\mathrm{coeffs}}^{(k)}$
\State $A(\tau) \gets -\exp(A_{\log\_\mathrm{base}}+\beta_A\tanh(A_{\mathrm{shift}}))$ \Comment{negative diagonal continuous dynamics}
\State $U(\tau) \gets U_{\mathrm{base}}+\beta_B\tanh(U_{\mathrm{shift}})$, \quad $V(\tau) \gets V_{\mathrm{base}}+\beta_B\tanh(V_{\mathrm{shift}})$
\State $B(\tau) \gets U(\tau)V(\tau)$, \quad $C \gets \mathbf{1}$ \Comment{fixed readout removes scaling ambiguity}
\State \textbf{Return} $A(\tau), B(\tau), C$
\end{algorithmic}
\end{algorithm}

The synthetic Chebyshev generator uses a degree-3 basis to keep the target low-dimensional, whereas the language-model Chebyshev adapter uses degree 5 (Table~\ref{tab:kernel_training}). The synthetic target is therefore a smoother, lower-capacity member of the same polynomial family.

The DCT generator replaces the Chebyshev basis with cosine modes and a learned monotone spectral decay. The normalization constants match the $L_{\mathrm{time}}=256$ DCT-II basis used in the experiment, while tanh-bounded shifts keep the recurrence and input map controlled. As in the Chebyshev setting, $C=\mathbf{1}$ is fixed so the comparison focuses on the recovered operator.

\begin{algorithm}[H]
\caption{Discrete Cosine Transform (DCT) Synthetic LPV-SSM}
\label{alg:dct_ssm}
\begin{algorithmic}[1]
\State \textbf{Input:} Diffusion coordinate $\tau \in [0,1]$
\State \textbf{Parameters:} $A_{\log\_\mathrm{base}}$, $U_{\mathrm{base}}$, $V_{\mathrm{base}}$, spectral coefficients $A_{\mathrm{spec}},U_{\mathrm{spec}},V_{\mathrm{spec}}$, bounds $\beta_A,\beta_B$, bandwidth $\alpha$, $n_{\mathrm{freqs}}=8$, $L_{\mathrm{time}}=256$
\For{$k = 0$ \textbf{to} $n_{\mathrm{freqs}}-1$}
    \State $w_k \gets (k+1)^{-(\mathrm{softplus}(\alpha)+10^{-4})}$ \Comment{learned spectral attenuation}
    \State $c_k \gets 1/\sqrt{L_{\mathrm{time}}}$ if $k=0$, else $\sqrt{2/L_{\mathrm{time}}}$ \Comment{DCT-II normalization}
    \State $b_k(\tau) \gets c_k \cos(\pi k\tau)$
\EndFor
\State $A_{\mathrm{shift}} \gets \sum_{k=0}^{n_{\mathrm{freqs}}-1} b_k(\tau)\,w_k\,A_{\mathrm{spec}}^{(k)}$ \Comment{attenuated spectral recurrence shift}
\State $U_{\mathrm{shift}} \gets \sum_{k=0}^{n_{\mathrm{freqs}}-1} b_k(\tau)\,w_k\,U_{\mathrm{spec}}^{(k)}$, \quad $V_{\mathrm{shift}} \gets \sum_{k=0}^{n_{\mathrm{freqs}}-1} b_k(\tau)\,w_k\,V_{\mathrm{spec}}^{(k)}$
\State $A(\tau) \gets -\exp(A_{\log\_\mathrm{base}}+\beta_A\tanh(A_{\mathrm{shift}}))$ \Comment{negative diagonal continuous dynamics}
\State $U(\tau) \gets U_{\mathrm{base}}+\beta_B\tanh(U_{\mathrm{shift}})$, \quad $V(\tau) \gets V_{\mathrm{base}}+\beta_B\tanh(V_{\mathrm{shift}})$
\State $B(\tau) \gets U(\tau)V(\tau)$, \quad $C \gets \mathbf{1}$ \Comment{fixed readout isolates operator recovery}
\State \textbf{Return} $A(\tau), B(\tau), C$
\end{algorithmic}
\end{algorithm}

The synthetic DCT generator keeps the basis length and frequency count of the language-model adapter so that the target uses the same bandlimited parameterization it is meant to probe. The learned attenuation only reshapes how quickly the high-frequency coefficients decay, so a miss in recovery is attributable to the bandlimited basis rather than to a mismatch in basis definition.

\end{document}